\documentclass[conference]{IEEEtran}

\usepackage{times}

\usepackage[numbers]{natbib}

\usepackage{cuted}
\usepackage{graphics} 
\usepackage{epsfig} 
\usepackage{times} 
\usepackage{amsmath} 
\usepackage{mathtools,amsfonts}
\usepackage{amssymb}  
\usepackage[T1]{fontenc}
\DeclareTextSymbolDefault{\dh}{T1}
\usepackage[bookmarks=true]{hyperref}
\usepackage{url}

\usepackage{algorithm}
\usepackage[noend]{algpseudocode}
\usepackage{newtxtext}
\usepackage{newtxmath}
\usepackage{lettrine}
\usepackage{booktabs}
\usepackage{multirow}
\usepackage{graphicx}
\usepackage{multicol}
\usepackage{caption}
\usepackage{subcaption}
\usepackage{mwe}
\usepackage{hyperref}
\usepackage{placeins}
\usepackage{svg}
\usepackage{tikz}
\usepackage{booktabs}
\usepackage{tablefootnote}
\usepackage{array}
    \newcolumntype{P}[1]{>{\centering\arraybackslash}p{#1}}
    \newcolumntype{M}[1]{>{\centering\arraybackslash}m{#1}}

\hypersetup{
    colorlinks=true,
    linkcolor=blue,
    filecolor=magenta,      
    urlcolor=cyan,
}

\makeatletter

\newtheorem{definition}{Definition}

\newtheorem{assump}{Assumption}

\newtheorem{theorem}{Theorem}

\makeatletter
\let\OldStatex\Statex
\renewcommand{\Statex}[1][3]{%
  \setlength\@tempdima{\algorithmicindent}%
  \OldStatex\hskip\dimexpr#1\@tempdima\relax}
\makeatother

\DeclarePairedDelimiter\abs{\lvert}{\rvert}%
\DeclarePairedDelimiter\norm{\lVert}{\rVert}%

\DeclareMathOperator*{\argmin}{arg\,min}

\makeatletter
\let\oldabs\abs
\def\abs{\@ifstar{\oldabs}{\oldabs*}}
\let\oldnorm\norm
\def\norm{\@ifstar{\oldnorm}{\oldnorm*}}
\makeatother

\definecolor{deep-red}{RGB}{192, 0, 0}
\definecolor{deep-purple}{RGB}{120, 0, 170}
\definecolor{good-green}{RGB}{0,175,0}
\definecolor{purple}{RGB}{210, 0, 210}
\definecolor{alizarin}{rgb}{0.82, 0.1, 0.26}

\newcommand{\xmark}{%
\tikz[scale=0.23] {
    \draw[line width=0.7,line cap=round] (0,0) to [bend left=6] (1,1);
    \draw[line width=0.7,line cap=round] (0.2,0.95) to [bend right=3] (0.8,0.05);
}}

\newcommand{\gcsnode}{\mathscr{Q}}
\newcommand{\lbgtogcs}{\mathcal{M}}

\usepackage[textsize=scriptsize, bordercolor=white,backgroundcolor=gray!30,linecolor=black,colorinlistoftodos]{todonotes}  
\usepackage[normalem]{ulem} 
\usepackage{soul} 

\usepackage[mathscr]{euscript}
\newif\ifTrackChanges   
\TrackChangesfalse 
\ifTrackChanges
    
    \newcommand{\NStrike}[1]{\textcolor{alizarin}{\st{#1}}}
    
    \newcommand{\HStrike}[1]{\textcolor{purple}{\st{#1}}}
    \newcommand{\cut}[1]{{\color{gray}{#1}}}
\else
    
    \newcommand{\NStrike}[1]{}
    \newcommand{\HStrike}[1]{}
    
    \newcommand{\cut}[1]{{}}
\fi

\makeatletter
\apptocmd{\@maketitle}{} 
\makeatother

\begin{document}

\title{Implicit Graph Search for \\ Planning on Graphs of Convex Sets}



\author{
    \IEEEauthorblockN{Ramkumar Natarajan\IEEEauthorrefmark{1}, Chaoqi Liu\IEEEauthorrefmark{2}, Howie Choset\IEEEauthorrefmark{1}, Maxim Likhachev\IEEEauthorrefmark{1}}
    \IEEEauthorblockA{\IEEEauthorrefmark{1}The Robotics Institute at Carnegie Mellon University
    \\\{rnataraj, choset, maxim\}@cs.cmu.edu}
    \IEEEauthorblockA{\IEEEauthorrefmark{2}Department of Computer Science at the University of Illinois at Urbana-Champaign
    \\chaoqil2@illinois.edu}
    \url{https://sbpl.github.io/ixg/}
}



%

\maketitle

\begin{strip}       
    \centering
    \vspace{-40pt}
    \begin{minipage}{\linewidth}
        \centering
        \includegraphics[width=\textwidth]{img/traj2_filmstrip_2x8.pdf}
    \end{minipage}
    
    \vspace{2.5pt}
    
    \begin{minipage}{\linewidth}
        \centering
        \includegraphics[width=\textwidth]{img/traj4_filmstrip_2x8.pdf}
    \end{minipage}
    \captionof{figure}{Smooth, collision-free motions generated by INSATxGCS (IxG) for various tasks with three Motoman HC10DTP arms operating simultaneously in a 18-DoF multi-arm assembly scenario.}
    \label{fig:moto} 
\end{strip}

\begin{abstract}

Smooth, collision-free motion planning is a fundamental challenge in robotics with a wide range of applications such as automated manufacturing, search \& rescue, underwater exploration, etc. Graphs of Convex Sets (GCS) is a recent method for synthesizing smooth trajectories by decomposing the planning space into convex sets, forming a graph to encode the adjacency relationships within the decomposition, and then simultaneously searching this graph and optimizing parts of the trajectory to obtain the final trajectory. To do this, one must solve a Mixed Integer Convex Program (MICP) and to mitigate computational time, GCS proposes a convex relaxation that is \textcolor{black}{empirically} very tight. Despite this tight relaxation, motion planning with GCS for real-world robotics problems translates to solving the simultaneous batch optimization problem that may contain millions of constraints and therefore can be slow. This is further exacerbated by the fact that the size of the GCS problem is invariant to the planning query. \textcolor{black}{Motivated by the} observation that the trajectory solution lies only on a fraction of the set of convex sets, we present two implicit graph search methods for planning on the graph of convex sets called INSATxGCS (IxG) and IxG*. INterleaved Search And Trajectory optimization (INSAT) is a previously developed algorithm that alternates between searching on a graph and optimizing partial paths to find a smooth trajectory. By using an implicit graph search method INSAT on the graph of convex sets, we achieve faster planning while ensuring stronger guarantees on completeness and optimality. Moveover, introducing a search-based technique to plan on the graph of convex sets enables us to easily leverage well-established techniques such as search parallelization, lazy planning, anytime planning, and replanning as future work. Numerical comparisons against GCS demonstrate the superiority of IxG across several applications, including planning for an 18-degree-of-freedom multi-arm assembly scenario.
\end{abstract}


\IEEEpeerreviewmaketitle

\section{Introduction}

\lettrine[]{T}{}rajectory optimization is widely used in motion planning in high-dimensional state spaces under kinodynamic constraints. However, complex configuration spaces cluttered with obstacles, as few as one, can make the optimization-based motion planning problem nonconvex \cite{fasttrajopt, chomp, trajopt, collopt}. The prevailing wisdom in optimization is to somehow transform the problem into a convex program, which can be solved efficiently for very large problems. Recently, graphs of convex sets \footnote{Per the authors of GCS \cite{sppgcs, gcs}, the graph of convex sets \textcolor{black}{is a directed graph in which each vertex is paired with a convex set whose spatial position is a continuous variable constrained to lie in the convex set and edge cost is a given convex function of the position of the vertices that this edge connects. Whereas,} GCS is the technique to plan on the graph of convex sets. \textcolor{black}{The terms ``GCS'' and ``graph of convex sets''} w\textcolor{black}{ere} used interchangeably in \cite{gcs} but \textcolor{black}{were} clear from the context and we follow the same in this paper.} \cite{gcs} decomposed the planning space into convex sets, and represented this decomposition as a graph (Fig. \ref{fig:gcsintro}). This decomposition has the potential to represent the planning space we encounter in most of the robotics motion planning problems as a sparse graph that covers most of the configuration space. Using this representation, motion planning on the graphs of convex sets introduced in \cite{gcs} called GCS is a remarkable formulation \cite{sppgcs} that transforms an inherently nonconvex problem into a convex optimization. GCS proposed a joint batch optimization over the entire graph (Fig \ref{fig:gcsvsixg}) to simultaneously choose the set of convex sets to traverse through and find the smooth trajectory via them. Such a giant optimization is conceivable as the research in convex optimization is very mature and is applied to efficiently solve large real-world problems. However, the joint batch optimization in GCS means that for any given task, the size of the GCS problem is independent of the planning query. It can potentially solve a massive optimization for every planning query. For example, for the challenging task of finding a trajectory to navigate through the clutter and finding a trajectory between two points that can be trivially connected in a large high-dimensional planning space, GCS attempts to solve the same-sized optimization program. This fundamentally limits the scope of its capability to plan fast for articulated systems, plan with a receding horizon, or do anytime planning. Consequently, we feel that the powerful representation of the graphs of convex sets is underutilized because of the choice of the solving strategy.

\begin{figure}
    \centering
    \includegraphics[width=\columnwidth]{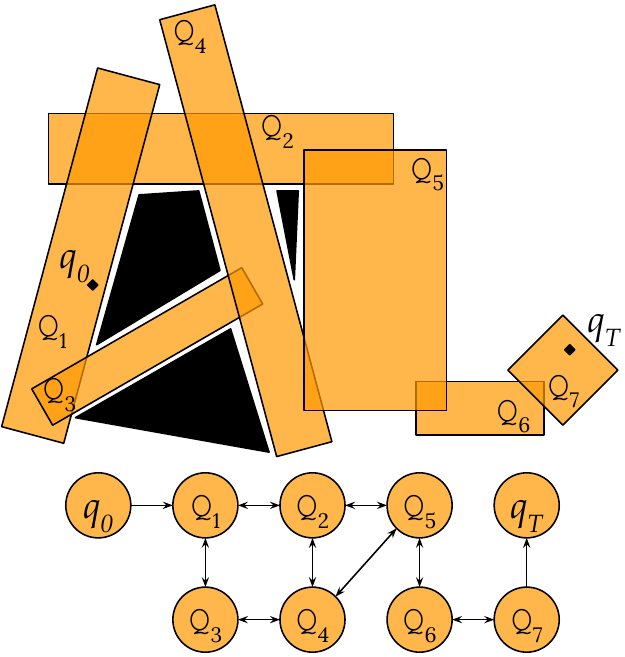}
    \caption{An example of decomposition of free space into convex sets $\gcsnode_1 \cup \ldots \cup \gcsnode_n$ with start $q_0$ and goal $q_T$ states (above) and their relationship represented as a graph (below).}
    \label{fig:gcsintro}
\end{figure}

\begin{figure*}%
\centering
\begin{subfigure}{.5\columnwidth}
\includegraphics[width=\columnwidth]{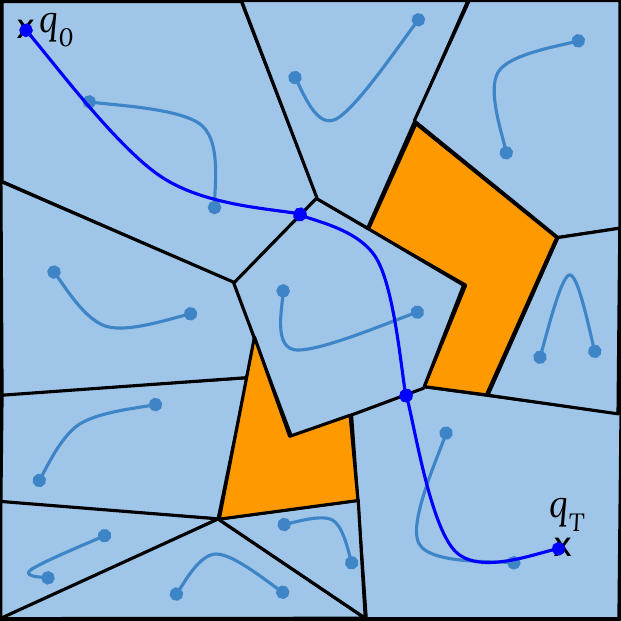}%
\caption{Joint GCS optimization}%
\label{sfa}%
\end{subfigure}
\begin{subfigure}{.5\columnwidth}
\includegraphics[width=\columnwidth]{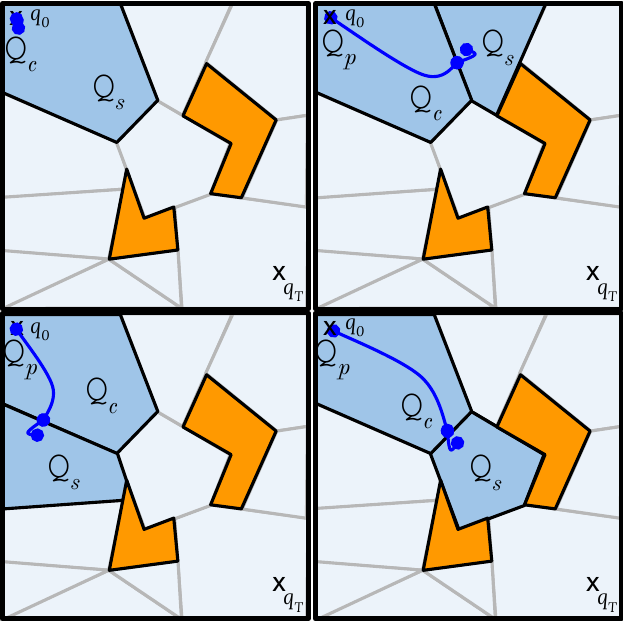}%
\caption{Expansion of start state}%
\label{sfb}%
\end{subfigure}
\begin{subfigure}{.5\columnwidth}
\includegraphics[width=\columnwidth]{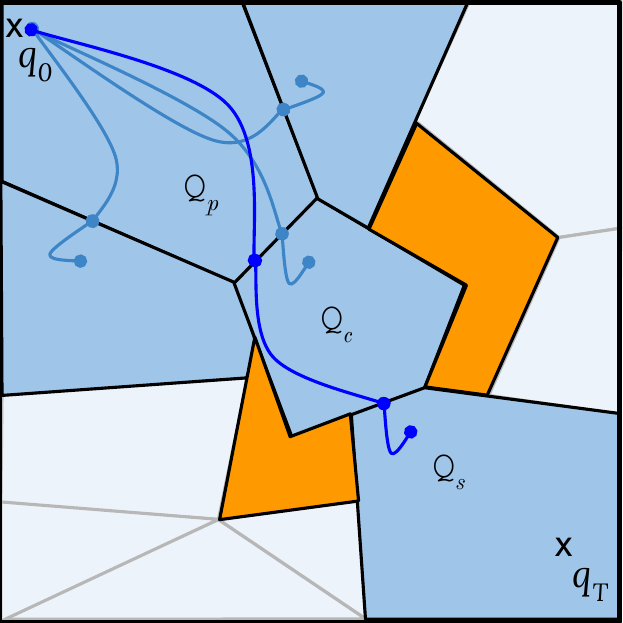}%
\caption{Reusing optimized trajectories}%
\label{sfc}%
\end{subfigure}
\begin{subfigure}{.5\columnwidth}
\includegraphics[width=\columnwidth]{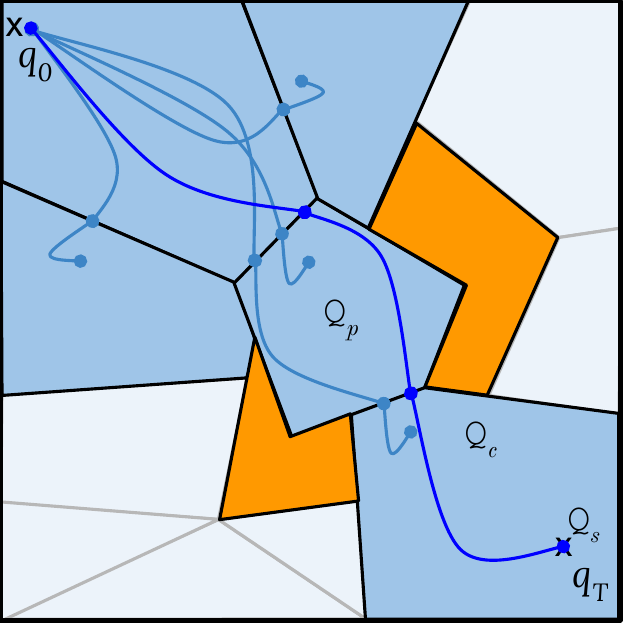}%
\caption{INSATxGCS solves faster}
\label{sfd}%
\end{subfigure}%
\caption{Graphical illustration of INSATxGCS in contrast to GCS. Fig. \ref{sfa} shows the joint batch optimization performed by GCS over the entire graph. The orange regions denote the obstacles and the blue regions denote the free space decomposed into a graph of convex sets. The final trajectory is shown as a dark blue curve and the optimization over each convex set of the graph is shown as light blue curves. Fig \ref{sfb}-\ref{sfd} shows the various steps of INSATxGCS incrementally exploring the convex sets and evaluating the edges only as needed. The convex sets in dark blue denote the explored graph and the light ones denote the unvisited convex sets. The trajectories in dark blue are the incremental optimizations that reuse the light blue solutions from previous expansions for computing trajectories to the successor states. INSATxGCS finds the solution trajectory to the goal without even visiting several convex sets of the graph (Fig. \ref{sfd}).}
\label{fig:gcsvsixg}
\end{figure*}

We see the aforementioned pitfalls as opportunities and this brings us to INSAT \cite{insat1,insat2, pinsat} which is a general framework that \textcolor{black}{combines the} best of graph search and trajectory optimization. It does so by establishing a symbiotic relationship between the discrete graph that chooses the optimizations to run and the trajectory optimization that decides which part of the discrete graph to explore. Given the sparsity of the GCS graph with respect to the dimension of the planning space and the fact that finding a trajectory via a unary tree of convex sets can be posed as a convex program, INSAT is naturally well-positioned to be a solver for planning on the graphs of convex sets. In this paper, we strongly advocate for using the implicit graph search-based technique INSAT to accelerate planning on the graphs of convex sets. We substantiate our claim by theoretically and experimentally showing that INSATxGCS (IxG) and IxG* \textcolor{black}{are} a superior alternative in all our test applications compared to GCS for synthesizing a smooth trajectory. We prove that IxG* has stronger guarantees than GCS on completeness and optimality. \textcolor{black}{IxG*'s stronger completeness enables finding solutions for challenging initial conditions (high velocities etc) and the stronger optimality provides the ability to prespecify the suboptimality bound of the solution as opposed to finding it after the fact in GCS.} We provide ample experimental evidence across different applications to show that IxG finds lower-cost solutions orders of magnitude faster than GCS.

The key idea in this work is based on the observation that depending on the planning query, we need not optimize over the entire graph of convex sets. To materialize this, we interleave searching on the GCS graph and running optimizations only on the fraction of the graph explored by the search. Due to the nature of this optimization in the explored partial graph \footnote{Note that the optimization here can be purely convex as opposed to it being a MICP or a convex relaxation in the batch formulation} and the systematic exploration by the search we are able to provide stronger guarantees on completeness and optimality. 
 We will now begin with a brief discussion of the relevant prior work, followed by our proposed method with its theoretical properties, and present the experimental results.

\section{Related Work}

Collision-free, kinodynamic planning has been a subject of extensive research, driven by the need for robots and autonomous systems to navigate complex environments while considering their dynamic and curvature constraints. This collision avoidance along with dynamic feasibility is typically achieved using one of the four different schemes below 
\begin{itemize}
    \item \textbf{Search-based} kinodynamic planning involves precomputing a set of short, dynamically feasible (smooth) trajectories that capture the robot's capabilities called motion primitives. Then search-based planning algorithms like A* and its variants \cite{maxprimitives} can be used to plan over a set of precomputed motion primitives. These search-based methods provide strong guarantees on optimality and completeness w.r.t the chosen motion primitives. However, the choice and calculation of these motion primitives that prove efficient can be challenging, particularly in high-dimensional systems.
    \item \textbf{Sampling-based} kinodynamic methods adapt and extend classic approaches such as Probabilistic Roadmaps (PRMs) \cite{prm} and Rapidly-exploring Random Trees (RRTs) \cite{rrt} to handle dynamic systems \cite{kinodynamic}. This is achieved using dynamically feasible rollouts with random control inputs, solutions of boundary value problems within the \textit{extend} operation, or using geometric primitives that satisfy curvature. There are probabilistically complete and asymptotically optimal variants \cite{asymprrt, completerrt}, however, empirical convergence might be tricky.
    \item \textbf{Optimization-based} planning methods can generate high-quality trajectories that are dynamically feasible and do not suffer from the curse of dimensionality in search-based methods. They formulate the motion planning problem as an optimization problem \cite{optctr} with cost functions defined for trajectory length, time, or energy consumption. After transcribing into a finite-dimensional optimization, these methods rely on the gradients of the cost function and dynamics and employ numerical optimization algorithms to find locally optimal solutions \cite{trajopt, ddp, komo}. However, except for a small subset of systems (such as linear or flat systems), for most nonlinear systems these methods lack guarantees on completeness, optimality or convergence. In addition to poor convergence and lack of guarantees, dealing with obstacle constraints in these methods are by far the most challenging compared to other options. To the best of our knowledge, GCS \cite{sppgcs, gcs} is the first method to turn the optimization among obstacles into a convex program. 
    \item \textbf{Hybrid} planning methods combine two or all of the above-mentioned schemes. Search and sampling methods are combined in \cite{bit, ss1, ss2}, sampling and optimization methods are combined in \cite{rabit, so1, so2}, search, sampling and optimization methods are combined in \cite{idba}. INSAT combined search and optimization methods and demonstrated its capability in several complex dynamical systems \cite{insat, insatptc, pinsat, insatshield, troptc}. \textcolor{black}{Inspired by INSAT, the approach presented in this paper belongs to this category and combines implicit graph search (Sec. \ref{sssec:impvsexp}) and convex trajectory optimization.}
\end{itemize}

\section{Problem Statement}
The basic problem statement is the same as motion planning on GCS \cite{gcs}. We only present a better way to solve this formulation. So for consistency, we borrow the abstract problem formulation from \cite{gcs}. Consider a set of convex sets $\gcsnode = \{\gcsnode_1, \gcsnode_2, \ldots, \gcsnode_n\}\subset \mathbb{R}^d$ that capture the free and safe planning space for the robot to navigate (Fig. \ref{fig:gcsintro}-top). These convex sets can be computed in \textcolor{black}{the Cartesian} space \cite{iris1} \textcolor{black}{or in the} configuration space for manipulators using \cite{ciris, irisnlp}. The adjacency relationship between these convex sets is represented as a graph $G_{\gcsnode}=(V_{\gcsnode}, E_{\gcsnode})$ (Fig. \ref{fig:gcsintro}-bottom). Given these convex sets $\gcsnode$, the goal of the planning algorithm is to find a trajectory $q: [0, T] \rightarrow \gcsnode$ that obeys the following optimization


\begin{subequations}
\begin{alignat}{2}
\text { minimize } & a L(q)+b T & \label{eq:obj1} \\
\text { subject to } & q(t) \in \gcsnode_1 \cup \ldots \cup \gcsnode_n, & \forall t \in[0, T] \label{eq:obj2} \\
& \dot{q}(t) \in \mathscr{D}, & \forall t \in[0, T] \label{eq:obj3} \\
& q(0)=q_0, q(T)=q_T & \label{eq:obj4} \\
& \dot{q}(0)=\dot{q}_0, \dot{q}(T)=\dot{q}_T \label{eq:obj5} &
\end{alignat}
\label{eq:obj}
\end{subequations}
\noindent
where $T$ and $L(q)$ are the duration and arc length of the trajectory $\int_0^T \norm{\dot{q}(t)}_2dt$, $a,b \geq 0$ are the weights corresponding to the relative importance of these terms (Eq. \ref{eq:obj1}). The constraints require that the trajectory lie entirely within the union of the decomposed convex sets (Eq. \ref{eq:obj2}), \textcolor{black}{respecting} the velocity limits (Eq. \ref{eq:obj3}) \textcolor{black}{given by the convex set $\mathscr{D}$ at all times $t$}, and satisfy the boundary conditions (Eq. \ref{eq:obj4}, \ref{eq:obj5}) provided by the start state $q_0$, goal state $q_T$ and their velocities $\dot{q}_0, \dot{q}_T$.




\section{Background}
\subsection{Preliminaries}
\subsubsection{Prescribed Graph} A graph $G=(V, E)$ is \textit{prescribed} if the values of vertices and edge weights are prespecified. This is the most common version of a graph. We explicitly introduce this terminology to distinguish it from a graph of convex sets where the vertices are convex sets and edge lengths are convex functions of continuous variables representing the position of the vertices \cite{sppgcs}. The combinatorial problem of finding the shortest path on a prescribed graph can be formulated as a simple linear program (LP) \cite{algtb}.

\subsubsection{Implicit vs Explicit Graph Search} Implicit graph search is a technique that interleaves the construction of the graph with the exploration of space. The performance of graph search algorithms is heavily dominated by the use of appropriate data structures. In many robotics applications, the memory requirement to store the entire graph of the state space explicitly is prohibitive and sometimes even impossible. These problems are solved tractably by interleaving the construction of the graph and searching over it. In the case of graphs of convex sets, although the construction is performed offline and the entire graph is explicitly stored, the trajectories via them are optimized during the search, thereby making the search itself implicit.
\label{sssec:impvsexp}

\textcolor{black}{\subsubsection{Lower Bound Graph (LBG)} LBG is a lightweight surrogate graph to the underlying graph of convex sets that can be constructed offline using problem-specific parameters. The purpose of LBG is to provide a provable underestimate on the optimal cost between any two states in the graph of convex sets. Once the start and goal states are provided, we run a cheap backward Dijkstra search from the goal state on LBG to generate an admissible heuristic that accelerates IxG and IxG*. The details for the construction of LBG are given in Sec. \ref{sec:lbg} and its usage within IxG and IxG* is provided in Sec. \ref{sec:ixg} and \ref{sec:ixgs}.}

\textcolor{black}{\subsubsection{Constructing Graph of Convex Sets} Given the environment, the vertices of the graph of convex sets are constructed by growing convex regions in the free planning space. These regions decompose the planning space and the methods used to grow these regions depend on the problem at hand. They can range from simple axis-aligned rectangles for maze-like environments \ref{fig:gcsopt_maze} to decomposing 3D obstacle-free spaces into convex polytopic and ellipsoidal regions \cite{iris1} and complex configuration spaces into convex hyper-polyhedrons and hyper-ellipsoids \cite{ciris, irisnlp}. There is also a recent extension to using the clique cover of visibility graph for constructing minimal convex sets with high coverage \cite{vcc}. In the context of collision-free motion planning, an edge between two regions is considered if they overlap.}

\subsection{Motion Planning on GCS using Convex Optimization}
We provide a brief background on the first work that introduced shortest path planning (SPP) on GCS \cite{gcs, sppgcs} using the flowchart in Fig. \ref{fig:gcs_flow}. The global motion planning problem of navigating around obstacles is nonconvex. This is because the free planning space that constitutes the domain of the motion planning search/optimization becomes nonconvex even with the presence of one obstacle. Consequently, one needs to solve a nonlinear program (NLP) such as Eq. \ref{eq:obj} to solve the motion planning problem. More specifically, since the trajectory $q(t)$ could lie on a finite subset of convex sets $\gcsnode_1, \ldots, \gcsnode_n$, the criteria to identify that subset can be transcribed as an integer constraint over the set in Eq. \ref{eq:obj2}. Precisely stated, this makes the transcription of optimization in Eq. \ref{eq:obj} into a mixed integer nonlinear program (MINLP). The approach in \cite{sppgcs} generalizes the idea of using LP for SPP in prescribed graphs to SPP in GCS using a mixed integer convex program (MICP). 


Planning on GCS is formulated as MICP where the mixed integer part optimizes for the set of overlapping convex sets to traverse through and the convex optimization computes the trajectory via the overlapping convex sets. In other words, from a graph search perspective, the integer program assigns zeros and ones to edges to decide the optimal edges connecting from start to goal. In GCS, both these subproblems are simultaneously solved in one single optimization. Finding a solution for MICP requires an exact branch and bound algorithm and can be inefficient to solve in practice. GCS proposes a convex relaxation on the edge indicator variables of this MICP and demonstrates that this relaxation is very tight in practice. Once the relaxed values on the edges for the edge variables and the parameters of the choice of trajectory representation within each convex set are determined, a simple depth-first search rounding is performed to extract the final trajectory. 

\begin{figure}
    \centering
    \includegraphics[width=\columnwidth]{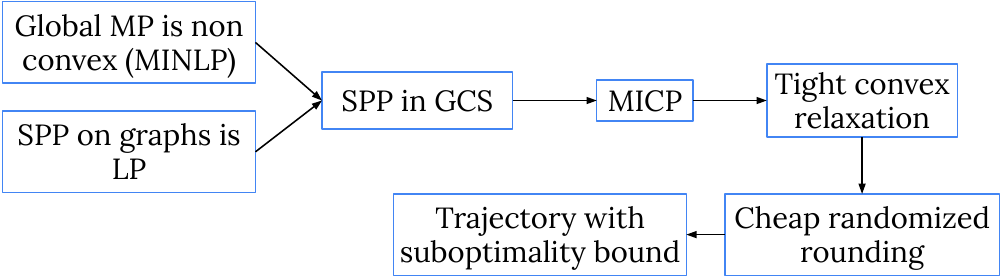}
    \caption{Working principle of GCS \cite{gcs}.}
    \label{fig:gcs_flow}
\end{figure}


\subsection{INSAT: INterleaved Search And Trajectory Optimization}

Grid search is an incredibly powerful tool for planning in state spaces with low dimensions. However, its efficacy diminishes as the dimensionality of the problem grows. In such cases, practitioners often opt for different planning approaches, such as sampling-based planning \cite{prm, rrt, kinodynamic, completerrt} and optimization-based planning \cite{optctr, trajopt, komo, ddp}. These alternatives, while valuable, each come with their own set of trade-offs. Sampling-based planning only provides probabilistic completeness; optimization-based planning demands substantial computational resources and is susceptible to getting trapped in local minima.

\begin{figure}
    \centering
    \includegraphics[width=0.7\columnwidth]{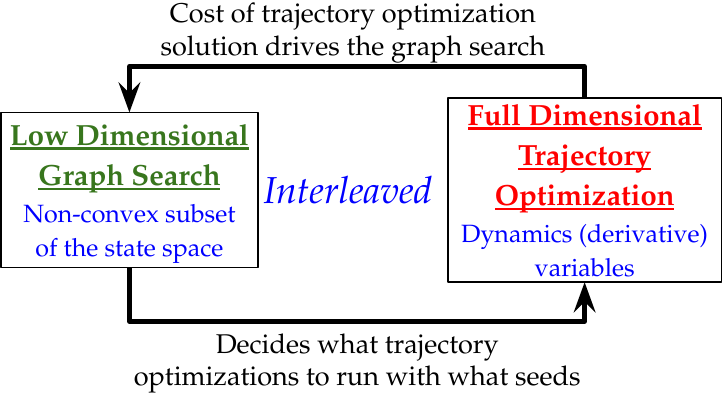}
    \caption{Working principle of INSAT}
    \label{fig:insat_schematic}
\end{figure}

The INSAT framework provides a mechanism for interleaving discrete search and continuous optimization that aims to address these challenges effectively (Fig. \ref{fig:insat_schematic}). The key idea behind INSAT is (a) to identify a low-dimensional manifold, (b) perform a search over a grid-based or sampling-based graph that represents this manifold, (c) while searching the graph, utilize high-dimensional trajectory optimization to compute the cost of partial solutions found by the search. As a result, the search over the lower-dimensional graph decides what trajectory optimizations to run and with what seeds, while the cost of solution from the trajectory optimization drives the search in the lower-dimensional graph until a feasible high-dimensional trajectory from start to goal is found. The dynamic programming in the low-dimensional search enables warm-starting trajectory optimizations using highly informative initial guesses.

\section{Motion Planning on GCS using INSAT}
Before we dive into our proposed methods, we will first describe in detail why the GCS representation is underutilized because of the choice of the solution strategy. Following that, we will present our approach of using INSAT as a solver for trajectory optimization on graphs of convex sets leading to two algorithms, IxG and IxG*. 
\subsection{GCS Representation is Underutilized}
\begin{figure}
    \centering
    \includegraphics[width=0.6\columnwidth]{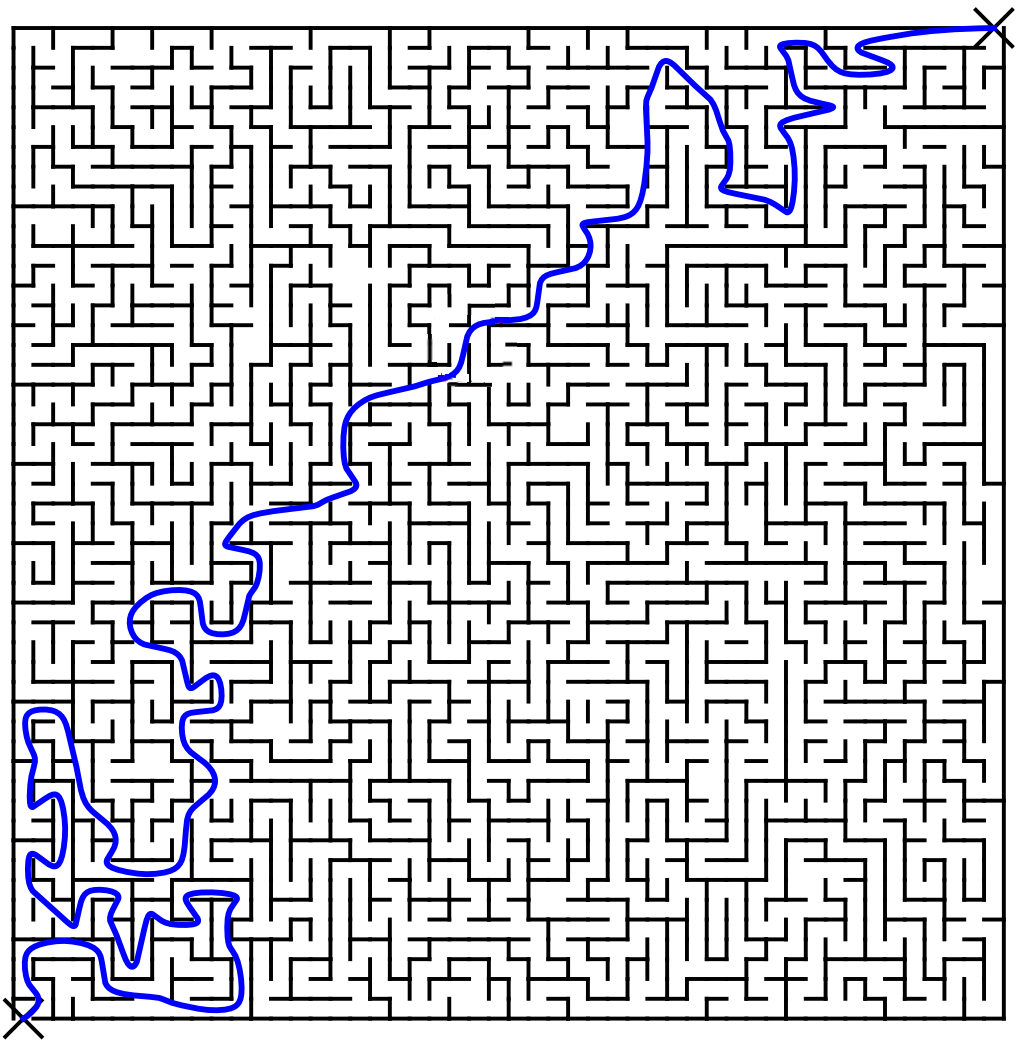}
    \caption{GCS optimization for planning in this maze (figure borrowed from \cite{gcs}) has 1,256,157 constraints regardless of the positions of $q_0$ and $q_T$.}
    \label{fig:gcsopt_maze}
\end{figure}
The compact joint optimization is elegant both in terms of the formulation and the result it generates. However, it falls short on four critical fronts namely, 
\begin{enumerate}
    \item As already mentioned, the optimization solves a huge problem even for a trivial planning query which is typically undesirable (Fig. \ref{fig:gcsopt_maze}).
    \item When planning for dynamical systems with \textcolor{black}{challenging initial conditions (with high velocities etc.)}, to guarantee completeness, the system might have to revisit a state and have a self-intersection in the configuration space (see Sec. \ref{sec:comp}). In the case of GCS, the states are convex sets which could be large subspaces in some instances. The current formulation of GCS does not permit trajectories that require revisiting the same convex set to reach the goal. A trivial modification of maintaining multiple copies of each convex set will make the already big GCS optimization more expensive.
    \item The suboptimality bound provided by the GCS method is only computed as an after-the-fact statistic. In other words, GCS does not provide a way to enforce a degree of suboptimality as part of the optimization. This option can be very useful in trading off the solution quality with the planning time in highly complex scenarios. 
    \item Though the convex decomposition of the planning space is represented using a graph data structure, because of the choice of solution strategy, the GCS method cannot leverage any of the graph search techniques such as anytime or incremental planning, planning with homotopy constraints, search parallelization using distributed hardware etc. 
\end{enumerate}

The method presented in this paper precisely addresses these limitations of GCS while having dramatically higher runtime performance. 

\begin{figure*}[htp!]
\centering
\begin{subfigure}{.48\textwidth}
\includegraphics[width=\columnwidth]{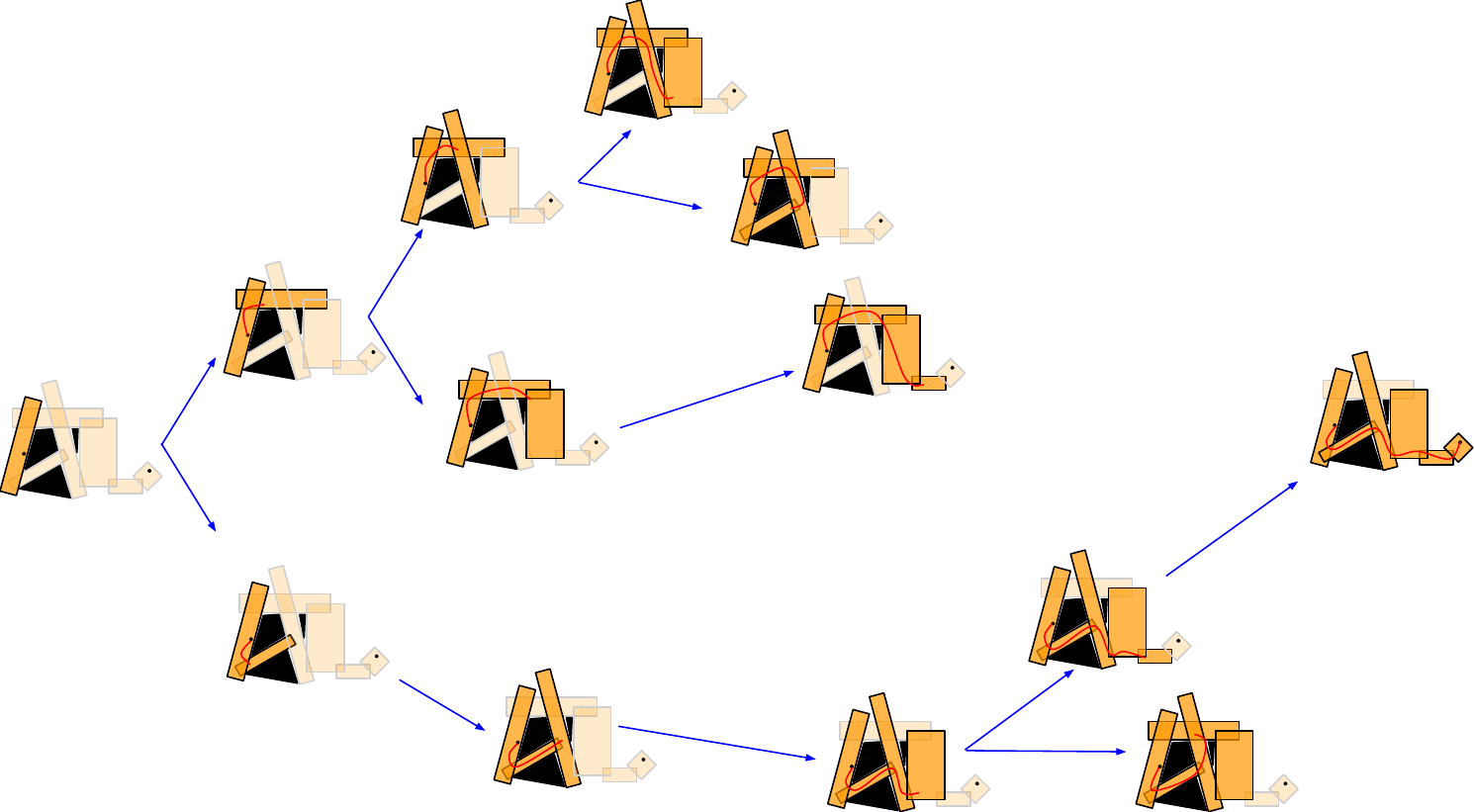}%
\caption{IxG search progression.}
\label{fig:ixg_graph}%
\end{subfigure}\quad
\begin{subfigure}{.48\textwidth}
\includegraphics[width=\columnwidth]{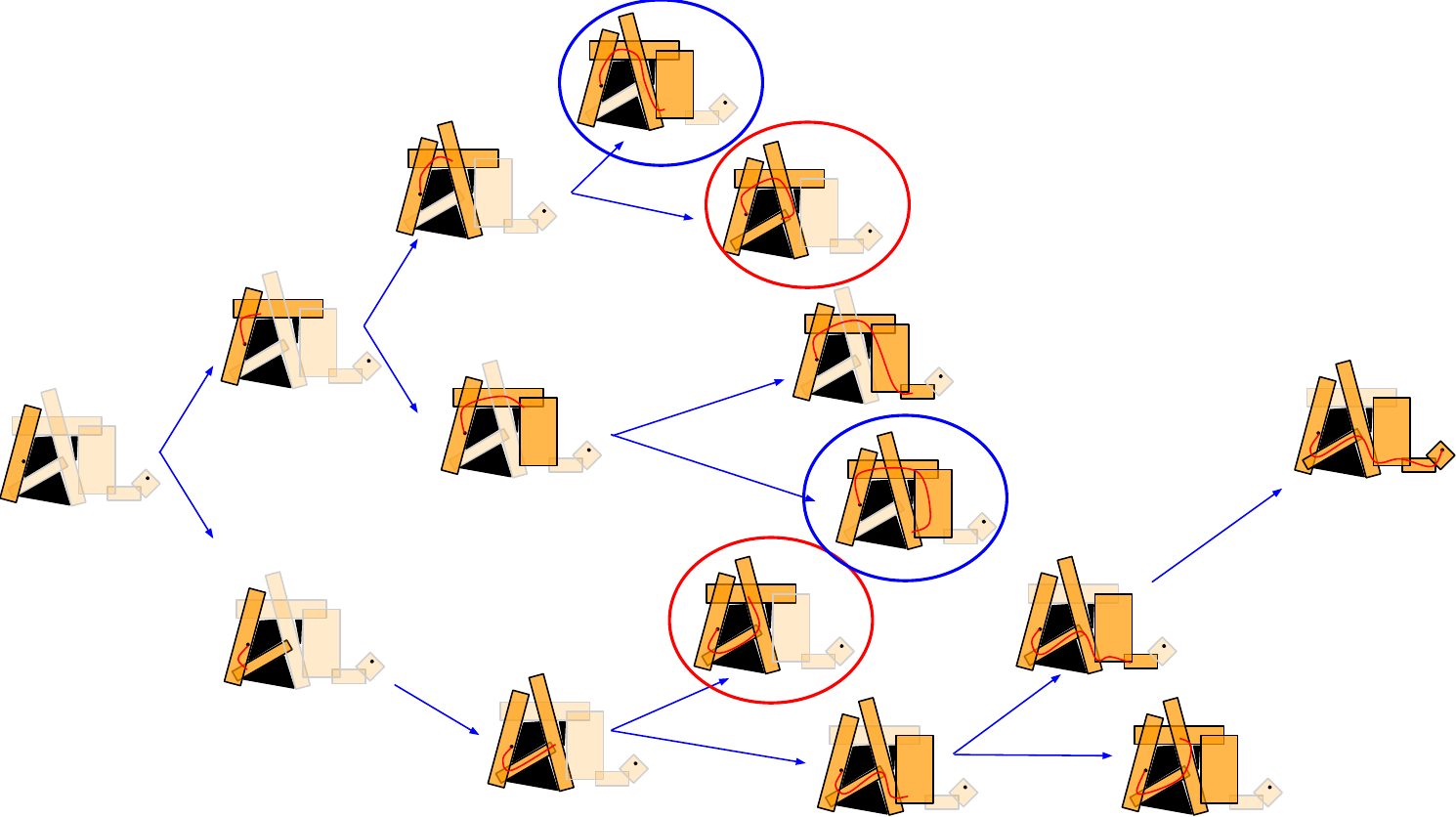}%
\caption{IxG* search progression.}
\label{fig:ixgs_graph}%
\end{subfigure}
\caption{Comparing IxG and IxG* on the same running example introduced in Fig. \ref{fig:gcsintro}. In IxG graph, the frontier consists of a successor region at most once. But in IxG* graph the same successor region via a different path is present multiple times shown using the same colored circles. This is because IxG* allows re-opening, re-expanding and revisiting states.}
\label{fig:ixgvsixgs}
\end{figure*}

\subsection{INSATxGCS (IxG)}
\label{sec:ixg}
We now explain IxG in detail using its pseudocode in Alg. \ref{alg:ixg} and visual illustrations in Fig. \ref{fig:gcsvsixg}, \ref{fig:ixgflow} and \ref{fig:ixgvsixgs}. As the relationship between the convex sets is represented using a graph data structure, a graph search algorithm is naturally suited to search over it for trajectory synthesis. Before the search begins, we construct the lower bound graph (LBG) $G_{lb}$ as explained below in Sec. \ref{sec:lbg}. The algorithm takes as input the start and the goal states $q_0$ and $q_T$, the graph of convex sets $G_\gcsnode$ and the LBG $G_{lb}$. It begins by assigning a convex set to $q_0$ and $q_T$ (Alg. \ref{alg:ixg}: line \ref{line:ixg_init}) and updating the LBG to reflect the current planner query with $q_0$ and $q_T$ (Alg. \ref{alg:ixg}: line \ref{fig:ixg_lbg1}-\ref{fig:ixg_lbg2}, explained in detail in Sec. \ref{sec:lbg}). Once the LBG is updated, we run an extremely fast backward Dijkstra search starting from the goal state $q_T$ on the LBG to obtain a provable lower bound from every node to $\gcsnode_T$ in $G_\gcsnode$. This lower bound by definition is an under-estimate on the optimal cost-to-go form any $\gcsnode_i$ and can be used as a heuristic to expedite the search \cite{astar}. As any informed graph search algorithm, IxG maintains a cost-to-come $g(\gcsnode_i)$ over every convex set in $G_\gcsnode$ and initializes $g(Q_0)=0$. Similar to a best first search like wA* \cite{pohlwastar}, IxG maintains a priority queue called OPEN over a list of convex sets to be expanded. 

\setlength{\textfloatsep}{4pt}
\begin{algorithm}
\begin{algorithmic}[1]

\Procedure{Key}{$\gcsnode$} \label{line:ixg_key} \Comment{\textcolor{black}{Computes priority value}}
\State \textbf{return} $g(\gcsnode) + \epsilon*l(\gcsnode)$ \Comment{Alg. \ref{alg:lbg}}
\EndProcedure

\Procedure{UpdateLBG}{$G_\gcsnode, G_{lb}, q_i$} \label{line:updatelbg} 
\State $\gcsnode^\prime = \{\gcsnode_i \in V_\gcsnode \mid q_i \in \gcsnode_i \}$ \Comment{$G_\gcsnode = (V_\gcsnode, E_\gcsnode)$} 
\State \textbf{for} $\gcsnode_i \in \gcsnode^\prime$ \textbf{do}
\State $\>$ \textbf{for} $v_{lb} \in G_{lb}$ \textbf{do}
\State $\>$ $\>$ \textbf{if} $v_{lb} \in \gcsnode_i$ \textbf{then}
\State $\>$ $\>$ $\>$ $G_{lb}.\text{AddVertex}(q_i)$ \Comment{See Fig. \ref{fig:lbg}, Sec. \ref{sec:lbg}}
\State $\>$ $\>$ $\>$ $G_{lb}.\text{AddEdge}((q_i, v_{lb}))$ \Comment{See Fig. \ref{fig:lbg}, Sec. \ref{sec:lbg}}
\State $\>$ $\>$ $\>$ $G_{lb}.\text{AddEdge}((v_{lb}, q_i))$ \Comment{See Fig. \ref{fig:lbg}, Sec. \ref{sec:lbg}}
\State \textbf{return} $G_{lb}$
\EndProcedure

\Procedure{Main}{$q_0, q_T, G_\gcsnode, {G}_{lb}$}
\State $\gcsnode_0 = q_0; \gcsnode_T = q_T$ \Comment{Convex sets for $q_0$ and $q_T$} \label{line:ixg_init}
\State $G_{lb}$ = \textproc{UpdateLBG}($G_\gcsnode, {G}_{lb}, q_0$) \Comment{See Fig. \ref{fig:lbg}} \label{fig:ixg_lbg1}
\State $G_{lb}$ = \textproc{UpdateLBG}($G_\gcsnode, {G}_{lb}, q_T$) \Comment{See Fig. \ref{fig:lbg}} \label{fig:ixg_lbg2}
\State $l(\gcsnode)$ = \textproc{Dijkstra}($G_{lb}, q_T$) \Comment{LB search given $q_T$} \label{line:ixg_dij}
\State $\forall \gcsnode_i \in V_\gcsnode, g(\gcsnode_i) = \infty$; $g(\gcsnode_0) = 0$ \label{line:ixg_ginit}\Comment{$G_\gcsnode = (V_\gcsnode, E_\gcsnode)$}
\State Insert $\gcsnode_0$ in OPEN with \textproc{Key}($\gcsnode_0$) \label{line:init}
\State \textbf{while} \textproc{Key}($\gcsnode_T$) $\le$ $\infty$ \textbf{do} \label{line:ixg_term}
\State $\>$ $\gcsnode_c =$ OPEN.pop() \label{line:pq} 
\State $\>$ $\gcsnode_p = $ \text{Predecessor}($\gcsnode_c$) \label{line:ixg_pred}
\State $\>$ \textbf{for} $\gcsnode_s \in$ \text{Successors}($\gcsnode_c$) \textbf{do}  \label{line:ixg_succ}
\State $\>$ $\>$ \textbf{if} $\gcsnode_s \notin $ CLOSED \textbf{then} \label{line:closed}
\State $\>$ $\>$ $\>$ $\gcsnode^{0\ldots c} = $ Ancestors($\gcsnode_s$)
\State $\>$ $\>$ $\>$ $\gcsnode^{0\ldots s} = (\gcsnode^{0\ldots c}, \gcsnode_s)$
\State $\>$ $\>$ $\>$ $q_{pcs}(t)$ = \textproc{LBGLookup}($\gcsnode_p, \gcsnode_c, \gcsnode_s$) \Comment{Sec. \ref{sec:lbg}} \label{line:ixg_lbglkp}
\State $\>$ $\>$ $\>$ $q_{0c}(t) = \gcsnode_c$.trajectory() \Comment{From recursion} 
\State $\>$ $\>$ $\>$ $q_{0s}(t)$ = \textproc{Optimize}($\gcsnode^{0\ldots s}, q_{0c}(t),q_{pcs}(t)$) \Comment{Eq. \ref{eq:pobj}} \label{line:ixg_opt}
\State $\>$ $\>$ $\>$ \textbf{if} $q_{0s}(t)$.isValid() \textbf{then} \label{line:ixg_valid}
\State $\>$ $\>$ $\>$ $\>$ \textbf{if} $c(q_{0s}(t)) < g(q_{0s}(t))$ \textbf{then} \Comment{Eq. \ref{eq:pobj1}}
\State $\>$ $\>$ $\>$ $\>$ $\>$ $\gcsnode_s$.trajectory = $q_{0s}(t)$
\State $\>$ $\>$ $\>$ $\>$ $\>$ Insert/Update $\gcsnode_s$ in OPEN with \textproc{Key}($\gcsnode_s$)
\State \textbf{return} $\gcsnode_T$.trajectory()
\EndProcedure
\end{algorithmic}
\caption{INSATxGCS (IxG)}
\label{alg:ixg}
\end{algorithm}

Inside the search loop that runs until the goal state is expanded (Alg. \ref{alg:ixg}: line \ref{line:ixg_term}), IxG picks the lowest cost node $\gcsnode_c$ for expansion \textcolor{black}{(Alg. \ref{alg:ixg}, line \ref{line:pq})} per the priority value (Alg. \ref{alg:ixg}: line \ref{line:ixg_key}). \textcolor{black}{The priority value is a sum of cost-to-come $g(\gcsnode)$ and weighted cost-to-go $\epsilon l(\gcsnode)$.} The term $\epsilon$ is an inflation on the admissible heuristic (underestimate) that trades off planning speed to solution quality \cite{pohlwastar}. It is important to note that the search also keeps track of the expanded nodes using CLOSED list (Alg. \ref{alg:ixg}, line \ref{line:closed}) to prevent re-expansion. As IxG is derived by applying INSAT to plan on GCS, for every successor $\gcsnode_s$ generated from expanding $\gcsnode_c$, we consider the set of ancestors $\gcsnode^{0\ldots s}$ and optimize a trajectory through them via a two-step process (Alg. \ref{alg:ixg}, lines \ref{line:ixg_succ}-\ref{line:ixg_opt}). \textcolor{black}{Let $\gcsnode_p$ be the predecessor of $\gcsnode_c$ (see Fig. \ref{fig:gcsvsixg}).} First, an incremental trajectory $q_{pcs}(t)$ that connects any point in $\gcsnode_p$ to any point $\gcsnode_s$ via $\gcsnode_c$ is computed. We will see in Sec. \ref{sec:lbg} how this step can fast-forwarded by a simple lookup. Then, the full trajectory from the start state $q_{0s}(t)$ is obtained by warm-starting an optimization over $\gcsnode^{0\ldots s}$ with $q_{0c}(t)$ and $q_{pcs}(t)$ (Sec. \ref{sec:seqopt}, Eq. \ref{eq:pobj}). Note that the dynamic programming nature of the algorithm will guarantee the existence of $q_{0c}(t)$, as otherwise $\gcsnode_c$ would not have been added to OPEN and therefore never expanded.  Finally, the optimized trajectory $q_{0s}(t)$ is checked for validity and added/updated in the OPEN list to be expanded in the future. When the goal convex region $\gcsnode_T$ is generated as a successor, if a valid trajectory to it $q_{0T}(t)$ is computed, then the condition in line $\ref{line:ixg_term}$ fails. The search then exits the loop and returns $q_{0T}(t)$.


\begin{figure}
    \centering
    \includegraphics[width=\columnwidth]{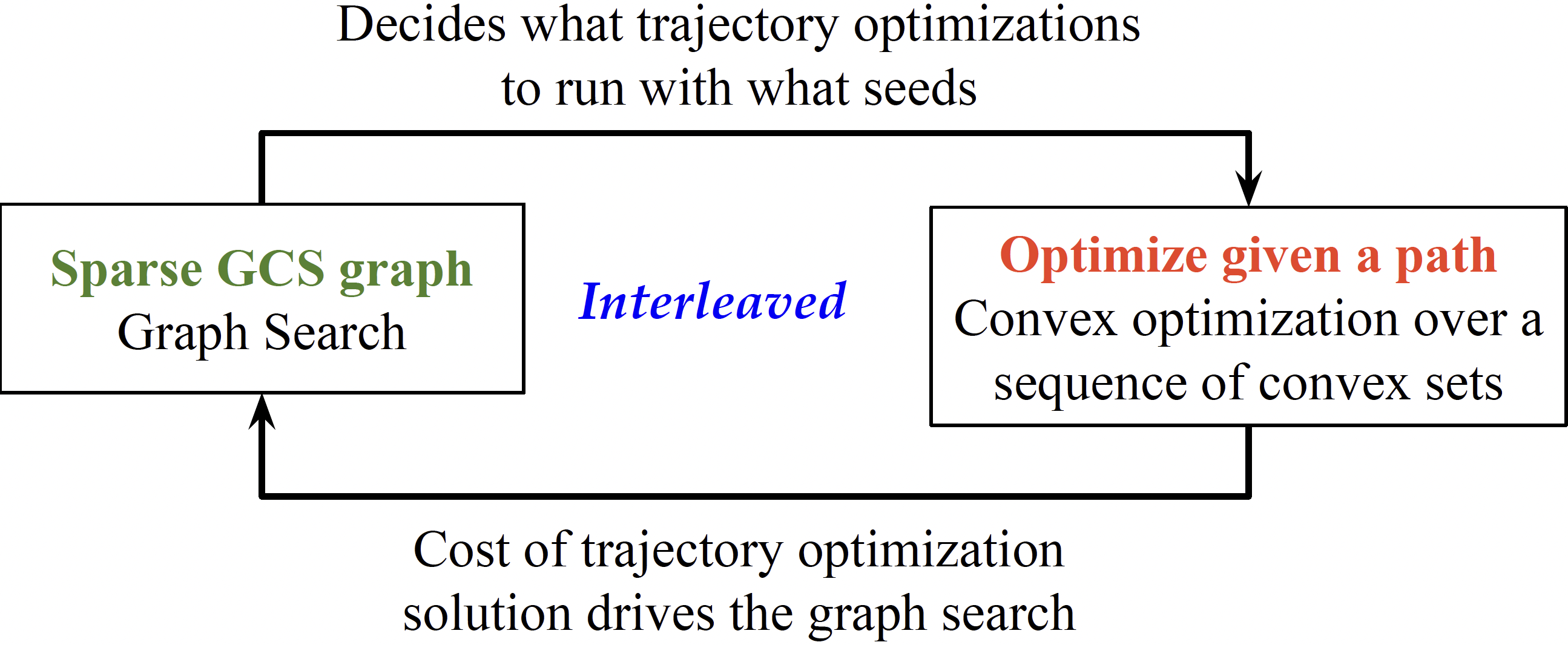}
    \caption{Working principle of IxG, IxG*}
    \label{fig:ixgflow}
\end{figure}

\subsubsection{Optimization over a Sequence of Convex Sets}
\label{sec:seqopt}
In line \ref{line:ixg_opt} of Alg. \ref{alg:ixg}, we solve an optimization given a sequence of convex sets, namely, the one obtained from the set of ancestors. This step is a crucial difference between the GCS trajectory optimization and IxG and perhaps the single most important reason for the efficiency of IxG. IxG decouples the joint GCS optimization of choosing the safe set of convex regions to traverse and designing robot trajectories within each region into searching over the safe set of convex regions, only optimizing trajectories over these partial paths from this search and guiding the search using the output of the optimized trajectory. As a result, the MICP in GCS trajectory optimization is split into a graph search that explores the convex regions systematically and a convex optimization over a prespecified sequence of convex sets.

Consider a prespecified path given by a sequence of convex regions $\gcsnode^{1\ldots K} = (\gcsnode_1, \gcsnode_2, \ldots, \gcsnode_K)$ in the graph $G_\gcsnode$ such that $(\gcsnode_k, \gcsnode_{k+1}) \in E_\gcsnode$ where $1<k \leq K \leq n$. The piecewise trajectory through the path $\gcsnode^{1\ldots K}$ made of a trajectory \textcolor{black}{segment} through each region $q_k(t)$ can be found by solving the below convex optimization. This is the operation carried out in line \ref{line:ixg_opt} of Alg. \ref{alg:ixg} and line \ref{line:ixgs_opt} of Alg. \ref{alg:ixgs}.

\begin{subequations}
\begin{alignat}{2}
\text { min } & a L(q)+b \sum_{k=1}^K T_k & \label{eq:pobj1} \\
\text { s.t. } & q_k(t_k) \in \gcsnode_k & \forall t_k \in[0, T_k], k \in {1,\ldots, K} \label{eq:pobj2} \\
& \dot{q_k}(t_k) \in \mathscr{D}  & \forall t_k \in[0, T_k], k \in {1,\ldots, K} \label{eq:pobj3} \\
& q^{(j)}_{k-1} (T_k) = q^{(j)}_{k}(0) \ \ & \forall t_k \in[0, T_k], k \in {2,\ldots, K} \label{eq:pobj4}
\end{alignat}
\label{eq:pobj}
\end{subequations}
\ignorespacesafterend
\textcolor{black}{where Eq. \ref{eq:pobj2} requires every trajectory piece $q_k(t_k)$ to lie in their corresponding convex set $\gcsnode_k$, Eq. \ref{eq:pobj3} is similar to Eq. \ref{eq:obj3} and Eq. \ref{eq:pobj4} captures overall continuity and smoothness using a boundary condition for each segment of the trajectory and its derivative such that the endpoint of a segment matches the starting point of the subsequent segment. In Alg. \ref{alg:ixg} and Alg. \ref{alg:ixgs} the start and goal states within the start and goal convex sets are represented as convex sets too (the start/goal state is a point in the ambient space, and is trivially a convex set). So the start state has an outgoing edge to the convex set it belongs to and the goal state has an incoming edge from the convex set it belongs to. Hence under this representation, the explicit boundary conditions similar to Eq. \ref{eq:obj4} and Eq. \ref{eq:obj5} can be removed in Eq. \ref{eq:pobj}.} 
\color{black}


\subsection{IxG*}
\label{sec:ixgs}

\setlength{\textfloatsep}{4pt}
\begin{algorithm}
\begin{algorithmic}[1]

\Procedure{Key}{$\gcsnode^{a\ldots b}$} \label{line:key}
\State \textbf{return} $g(\gcsnode^{a\ldots b}) + \epsilon * l(\gcsnode_b)$ \Comment{Get $\gcsnode_b$ from $\gcsnode^{a\ldots b}$}
\EndProcedure

\Procedure{UpdateLBG}{$\gcsnode, G_{lb}, q_i$} \label{line:updatelbg}
\State Same as Alg. \ref{alg:ixg}
\EndProcedure

\Procedure{Main}{$q_0, q_T, \gcsnode, {G}_{lb}$}
\State Same as lines \ref{line:ixg_init}-\ref{line:ixg_dij} of Alg. \ref{alg:ixg}
\State $q_{0T}(t) =$\textproc{IxG}($q_0, q_T, \gcsnode, {G}_{lb}$) \Comment{Use Alg. \ref{alg:ixg} for UB} \label{line:ixgs_ub1}
\State $u = \epsilon * c(q_{0T}(t))$ \Comment{UB on bounded suboptimal cost}  \label{line:ixgs_ub2}
\State $\forall \gcsnode^{0\ldots T}, g(\gcsnode^{0\ldots T}) = \infty$; $g(\gcsnode^0) = 0$ \label{line:ixgs_ginit}
\State Insert $\gcsnode^0$ in OPEN with \textproc{Key}($\gcsnode^0$) \label{line:init}
\State \textbf{while} \textproc{Key}($\gcsnode^{0\ldots T}$) $\le$ OPEN.min() \textbf{do} \label{line:ixgs_term}
\State $\>$ $\gcsnode^{0 \ldots c} =$ OPEN.pop() \label{line:ixgs_pq} \label{line:ixgs_pop}
\State $\>$ $\gcsnode_p = $ \text{Predecessor}($\gcsnode_c$)  \Comment{Get $\gcsnode_c$ from $\gcsnode^{0\ldots c}$}
\State $\>$ \textbf{for} $\gcsnode_s \in$ \text{Successors}($\gcsnode_c$) \textbf{do}  \label{line:ixgs_succ} \Comment{Get $\gcsnode_c$ from $\gcsnode^{0\ldots c}$}
\State $\>$ $\>$ \textbf{if} Allow Cycles \textbf{or} $\gcsnode_s \notin \gcsnode^{0\ldots c}$ \textbf{then} \label{line:ixgs_cycle}
\State $\>$ $\>$ $\>$ $\gcsnode^{0\ldots s} = (\gcsnode^{0\ldots c}, \gcsnode_s)$
\State $\>$ $\>$ \textbf{else if} $\gcsnode_s \in \gcsnode^{0\ldots c}$ \textbf{then}
\State $\>$ $\>$ $\>$ \textbf{continue}
\State $\>$ $\>$ $q_{pcs}(t)$ = \textproc{LBGLookup}($\gcsnode_p, \gcsnode_c, \gcsnode_s$) \Comment{Sec. \ref{sec:lbg}}\label{line:ixgs_lbglkp}
\State $\>$ $\>$ $q_{0c}(t) = \gcsnode^{0 \ldots c}$.trajectory() \Comment{From recursion} 
\State $\>$ $\>$ $q_{0s}(t)$ = \textproc{Optimize}($\gcsnode^{0\ldots s}, q_{0c}(t),q_{pcs}(t)$) \Comment{Eq. \ref{eq:pobj}} \label{line:ixgs_opt}
\State $\>$ $\>$ \textbf{if} $q_{0s}(t)$.isValid() \textbf{then} \label{line:ixgs_valid}
\State $\>$ $\>$ $\>$ \textbf{if} $c(q_{0s}(t)) + \epsilon*l(\gcsnode_s) > \epsilon*u$ \textbf{then} \label{line:ixgs_prune}
\State $\>$ $\>$ $\>$ $\>$ \textbf{continue} \Comment{Prune the path $\gcsnode^{0\ldots s}$}
\State $\>$ $\>$ $\>$ \textbf{else}
\State $\>$ $\>$ $\>$ $\>$ $\gcsnode^{0\ldots s}$.trajectory = $q_{0s}(t)$
\State $\>$ $\>$ $\>$ $\>$ Insert/Update $\gcsnode^{0\ldots s}$ in OPEN with \textproc{Key}($\gcsnode^{0\ldots s}$)
\State \textbf{return} $\gcsnode^{0\ldots T}$.trajectory()
\EndProcedure
\end{algorithmic}
\caption{IxG*}
\label{alg:ixgs}
\end{algorithm}

Before we describe the details of the optimal version of IxG called IxG*, we will briefly explain why IxG is not an optimal algorithm. An important caveat in considering the set of ancestors to optimize for a trajectory edge to every successor during expansion is that it breaks the Markov property of heuristic search. The Markov property in graph search requires that the cost of the successor depends only on the current state and not on the history of states leading up to it. Heuristic search methods leverage this property to introduce CLOSED list and guarantee optimality under admissible heuristics without re-opening and re-expanding states. As IxG violates the Markov property, it is neither optimal nor bounded suboptimal. Consequently, we need to allow the re-expansion of the convex regions in the search and essentially perform a tree search. Just by allowing the re-expansion of the states (with and without allowing cycles) and thereby searching over all possible paths for every single state, IxG can be made provably optimal and bounded suboptimal. However, this will blow up the number of expansions and make the algorithm intractable. To alleviate this, we introduce a pruning mechanism (Alg. \ref{alg:ixgs}, line \ref{line:ixgs_prune}) using a lower bound on the cost-to-go from every state (we use LBG computed in Alg. \ref{alg:lbg}) and an upper bound on the optimal cost (Alg. \ref{alg:ixgs}, line \ref{line:ixgs_ub1}-\ref{line:ixgs_ub2}). 

IxG* algorithm presented using pseudocode in Alg. \ref{alg:ixgs} inherits most of its procedure from IxG (Alg. \ref{alg:ixg}). The key differences are (i) the search is carried over the paths instead of convex states (see line \ref{line:ixgs_pop} in Alg. \ref{alg:ixgs}) (ii) the CLOSED list is omitted in favor of allowing re-opening, re-expanding and revisiting (allow cycles in line \ref{line:ixgs_cycle}) convex regions and (iii) the pruning criteria employed in line \ref{line:ixgs_prune}. The pruning step checks if the sum of the cost-to-come via a particular path and the cost-to-go obtained from LBG is greater than the maximum allowed cost computed as upper bound at the beginning of the search. If this condition is satisfied, then such a path will never lead to an optimal or $\epsilon$-suboptimal solution. As we are searching different paths of convex sets leading to the same convex region, the OPEN list contains a list of paths.

\subsection{Lower Bound Graph (LBG)}
\label{sec:lbg}
\begin{figure}
    \centering
    \includegraphics[width=\columnwidth]{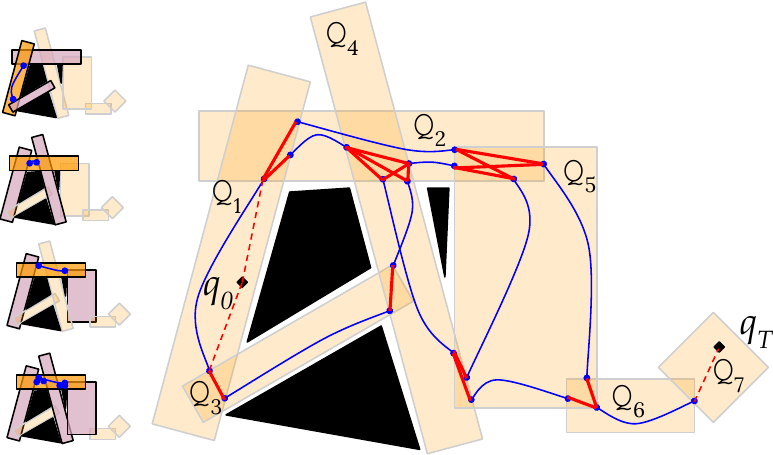}
    \caption{The lower bound graph (LBG) whose vertices are shown as blue dots and edges are shown using blue curves and red lines. The blue curves are found by solving a convex program over every convex set triplet connected by edges (shown in the left column with the triplet highlighted as purple-orange-purple). The red edges are optimized with simpler constraints or can be set to zero cost to satisfy the lower bound trivially. See Alg. \ref{alg:lbg} for construction of LBG.}
    \label{fig:lbg}
\end{figure}

The lower bound graph (LBG) is the core accelerator behind the efficient performance of IxG*. Given a choice of motion generation parameters such as the order, minimum derivative of continuity, etc in the case of B-splines trajectories, the LBG is a surrogate \textit{prescribed graph} computed over $G_\gcsnode$ using Alg. \ref{alg:lbg}. The edges of the LBG are formed by optimizing over a 3-sequence (triplet) of convex sets in $G_\gcsnode$ (Alg. \ref{alg:lbg}, line \ref{line:lbg_opt1}-\ref{line:lbg_opt2}). Each triplet constitutes a node in $G_\gcsnode$ along with its incoming and outgoing convex regions (Alg. \ref{alg:lbg}, lines \ref{line:lbg_oloop}-\ref{line:lbg_tlet}). These edges are then added to the LBG graph $G_{lb}$. The boundary points of the optimized edge trajectories form $V_{lb}$. Following the edges generated from triplet optimization (see blue curves in Fig. \ref{fig:lbg}), the graph is made into a single connected component by joining the terminals of optimized trajectories in the overlapping subset of convex regions with zero cost edges (Alg. \ref{alg:lbg}, lines \ref{line:lbg_olap}-\ref{line:lbg_zedge}) or edges with lower order and continuity (to certify provable lower bound). This is shown as red lines in Fig. \ref{fig:lbg}. Finally, the Alg. \ref{alg:lbg} returns the constructed $G_{lb}$ with a single component. 

The LBG graph is used for two purposes in IxG/IxG*. Using an optimal search such as Dijkstra over $G_{lb}$ gives a provable lower bound on the optimal cost to traverse between any two convex regions in $G_\gcsnode$. Given a planning query, $q_0$ and $q_T$, the LBG is used for computing the admissible heuristic from any state to $q_T$ and verifying the pruning criteria in IxG*. Moreover, since the LBG is formed by optimizing over the set of possible triplets in $G_\gcsnode$, they can be used directly as the solutions of incremental optimization steps in IxG and IxG* (Alg. \ref{alg:ixg}, line \ref{line:ixg_lbglkp} and Alg. \ref{alg:ixgs}, line \ref{line:ixgs_lbglkp})

\begin{figure}
    \centering
    \includegraphics[width=0.7\columnwidth]{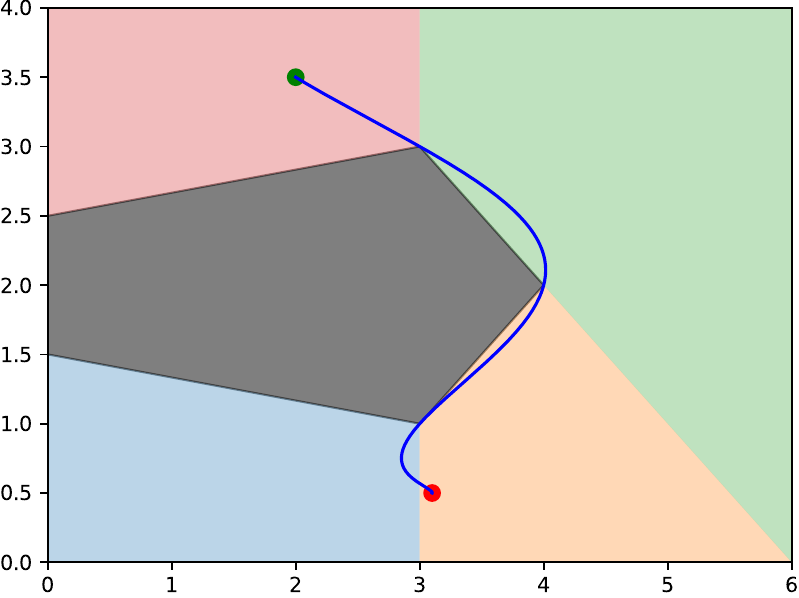}
    \caption{A simple 2D example with one obstacle (grey) illustrates the need to revisit a convex set to find a smooth trajectory from the start (red dot) to the goal (green dot). The free space is decomposed into four convex sets. In this example, the initial velocity at the start state in the orange convex set is set to a high nonzero value in the direction of the blue convex set to its left. As there are constraints on the smoothness of the trajectory at the convex set boundaries, the only solution is to exit the start orange set, enter the blue set, exit the blue set, and enter the orange set again before arriving at the goal. Finding this trajectory requires re-expanding the orange convex set during search which is not addressed in the original GCS trajectory optimization \cite{gcs}.}
    \label{fig:revisit}
\end{figure}

\subsection{Theoretical Analysis} 
\subsubsection{Stronger Completeness}
\label{sec:comp}
The stronger completeness discussed in this section is easier to appreciate in the context of dynamical systems. Consider a point robot with a high non-zero velocity starting near the edge of a convex set facing the direction of the neighboring convex set (Fig. \ref{fig:revisit}). To guarantee smoothness, the robot might have to enter the neighboring convex set and revisit the current convex set and this may have to be done several times depending on the environment. The current GCS formulation does not provide a way to handle this scenario and a trivial modification of maintaining multiple copies of each convex set will make the already big GCS optimization further expensive in addition to not guaranteeing completeness \footnote{The exact number of times a region has to be revisited cannot be exactly known ahead of time.}.

To combat this, the graph search underpinning the IxG* does not require prespecifying the number of times a convex region must be visited to find a solution and guarantee completeness. Instead, as the graph is built implicitly, the number of cycles can be adapted on demand and as per the need (Alg. \ref{alg:ixgs}, line \ref{line:ixgs_cycle}).  

\setlength{\textfloatsep}{4pt}
\begin{algorithm}
\begin{algorithmic}[1]

\Procedure{LBG}{$G_\gcsnode$} 
\State $G_{lb} = (V_{lb}, E_{lb}); V_{lb} = E_{lb} = \emptyset$
\State \textbf{for} $\gcsnode_c \in V_\gcsnode$ \textbf{do} \Comment{$G_\gcsnode = (V_\gcsnode, E_\gcsnode)$}
\State $\>$ \textbf{for} $\gcsnode_{p} \in$ Neighbors($\gcsnode_c$) \textbf{do} \label{line:lbg_oloop}
\State $\>$ $\>$ \textbf{for} $\gcsnode_{s} \in$ Neighbors($\gcsnode_c$) \textbf{do}
\State $\>$ $\>$ $\>$ \textbf{if} $\gcsnode_p \neq \gcsnode_c$ \textbf{then}
\State $\>$ $\>$ $\>$ $\>$ $\gcsnode^\dag = (\gcsnode_p, \gcsnode_c, \gcsnode_s)$
\State $\>$ $\>$ $\>$ $\>$ $\gcsnode^\ddag = (\gcsnode_s, \gcsnode_c, \gcsnode_p)$\label{line:lbg_tlet}
\State $\>$ $\>$ $\>$ $\>$ $q_{pcs}(t)$ = \textproc{Optimize}($\gcsnode^\dag$) \Comment{Eq. \ref{eq:pobj}}\label{line:lbg_opt1}
\State $\>$ $\>$ $\>$ $\>$ $q_{scp}(t)$ = \textproc{Optimize}($\gcsnode^\ddag$) \Comment{Eq. \ref{eq:pobj}}\label{line:lbg_opt2}
\State $\>$ $\>$ $\>$ $\>$ $V_{lb}$.Add($q_{pcs}(0)$); $V_{lb}$.Add($q_{pcs}(T)$)
\State $\>$ $\>$ $\>$ $\>$ $V_{lb}$.Add($q_{scp}(0)$); $V_{lb}$.Add($q_{scp}(T)$)
\State $\>$ $\>$ $\>$ $\>$ $E_{lb}$.Add($q_{pcs}(t)$) with cost $c(q_{pcs}(t))$ 
\State $\>$ $\>$ $\>$ $\>$ $E_{lb}$.Add($q_{scp}(t)$) with cost $c(q_{scp}(t))$ 
\State \textbf{for} $(\gcsnode_1, \gcsnode_2) \in E_\gcsnode$ \textbf{do} \Comment{$G_\gcsnode = (V_\gcsnode, E_\gcsnode)$} \label{line:lbg_olap}
\State $\>$ \textbf{for} $q_1, q_2 \in V_{lb}$ such that $q_1\neq q_2$ \textbf{do}
\State $\>$ $\>$ \textbf{if} $q_1, q_2 \in \gcsnode_1 \cap \gcsnode_2$ \textbf{then}
\State $\>$ $\>$ $\>$ $E_{lb}$.Add(($q_1, q_2$)) with cost 0 \Comment{LB edge} \label{line:lbg_zedge}
\State \textbf{return} $G_{lb}$
\EndProcedure
\end{algorithmic}
\caption{Lower Bound Graph (LBG) search}
\label{alg:lbg}
\end{algorithm}

\subsubsection{Trivial Parallelization of IxG*} Since the graph search component of IxG* searches over paths of convex sets and maintains a separate copy of every way to reach a particular convex set, there is no interdependence between nodes picked for expansion. As a result, depending on the thread budget, any number of nodes can be popped from the OPEN list in line \ref{line:ixgs_pop} of Alg. \ref{alg:ixgs} for concurrent expansion. The optimality and completeness of the algorithm will not be sacrificed so long as the order of expansion and the priority value are maintained. 

\subsubsection{Properties of IxG, IxG* and LBG}
Table. \ref{tab:tprop} lists the properties of optimality and completeness of IxG and IxG*. Detailed proofs of the properties are provided in Appendix 1 of the supplementary material. As noted before, the optimality and bounded suboptimality properties are stronger than what was provided for the GCS trajectory optimization \cite{gcs}. This is because IxG* enforces a factor of suboptimality of the solution in Alg. \ref{alg:ixgs} whereas it is only calculated as an after-the-fact statistic in \cite{gcs}. As explained above, IxG* also satisfies a stricter definition of completeness.

In the same vein, the proofs on the bound on the number of nodes in LBG and guarantees of generating provably admissible heuristic is provided in Appendix 2 of the supplementary material.  
\begin{table}[h!]
\centering
\begin{tabular}{c|ccc}
\textbf{} & \textbf{Optimality} & \textbf{Bounded Suboptimality} & \textbf{Completeness} \\ \hline
\textbf{IxG}  & \xmark  & \xmark & \checkmark\tablefootnote{Under the assumption that all the edge constraints can be satisfied}              \\
\textbf{IxG*} & \checkmark & \checkmark & \checkmark             \\
\end{tabular}
\caption{Optimality and completeness of IxG and IxG*.}
\label{tab:tprop}
\end{table}

\section{Experimental Results}
We evaluate the empirical performance of IxG and IxG* in simulation against GCS in three different applications: (1) a 2D system in a $50\times50$ maze environment, (2) UAV in a 3D highly cluttered environment with trees and buildings and (3) an assembly cell with three Motoman HC10DTP arms for assembly. All the methods are implemented in C++ in the backend and sometimes invoked from Python wrappers. Tests ran on a 128-core AMD Ryzen Threadripper Pro with 512GB of memory.

\subsection{2D Maze}
For the 2D maze environment, we utilized the same maze provided in \cite{gcs} and tested using 50 randomly sampled start-goal pairs. \textcolor{black}{The graph of convex sets is constructed by considering axis-aligned rectangles between the walls of the maze and their overlaps.} The goal of the planner is to generate a trajectory with $C^2$ continuity from start to goal. In this environment, there are 2500 convex sets and 5198 edges in the graph. The results show that IxG* consistently outperforms GCS in terms of planning time. However, with $\epsilon=6$, the solution cost is slightly larger than that of GCS (Table. \ref{tab:maze}). This is exactly the kind of trade-off between solution quality and planning time that can be obtained by controlling and pre-specifying the suboptimality factor. 

\begin{table}[htp!]
\centering
\begin{tabular}{c|cc}
\textbf{}                   & \textbf{GCS} & \textbf{IxG*} $(\epsilon=6)$ \\ \hline
\textbf{Success Rate (\%)}  & 100\%          & 100\%              \\
\textbf{Solution Cost}      & 52.636         & 52.832             \\
\textbf{Planning Time (s)}  & 6.728          & 1.2613              \\
\textbf{\# Optimized Edges} & 5198           & 440.94            
\end{tabular}
\caption{Various statistics show IxG* outperforming GCS in the 2D maze environment. Note that the higher solution cost of IxG* is because of planning with $\epsilon=6$.}
\label{tab:maze}
\end{table}

\begin{figure*}[htp!]
\centering
\begin{subfigure}{.5\columnwidth}
\includegraphics[width=\columnwidth]{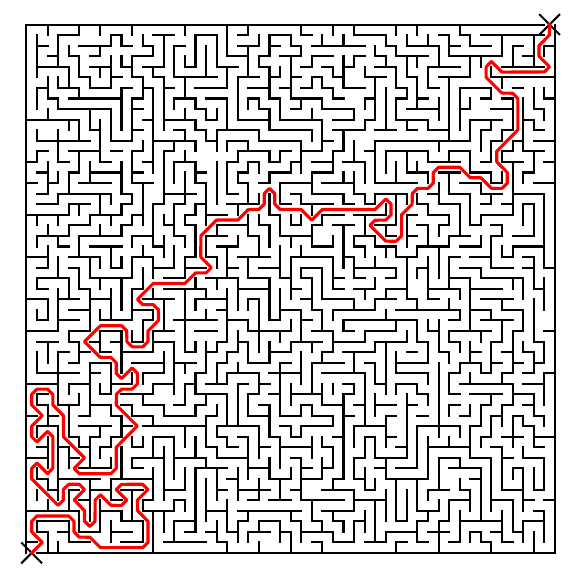}%
\caption{$\epsilon$ = 4}%
\label{m1}%
\end{subfigure}
\begin{subfigure}{.5\columnwidth}
\includegraphics[width=\columnwidth]{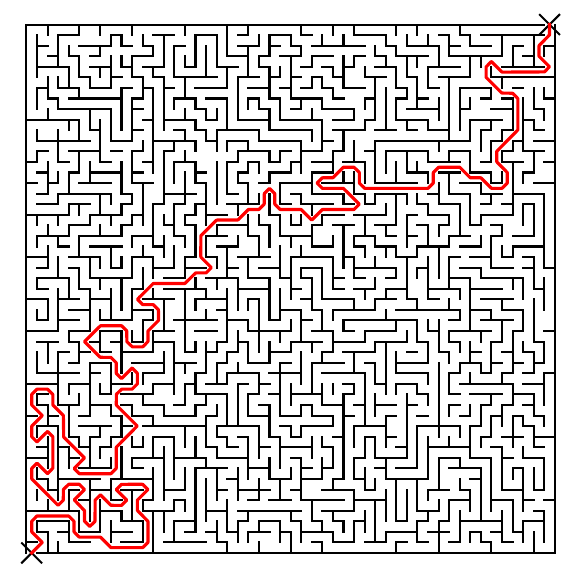}%
\caption{$\epsilon$ = 8}%
\label{m2}%
\end{subfigure}
\begin{subfigure}{.5\columnwidth}
\includegraphics[width=\columnwidth]{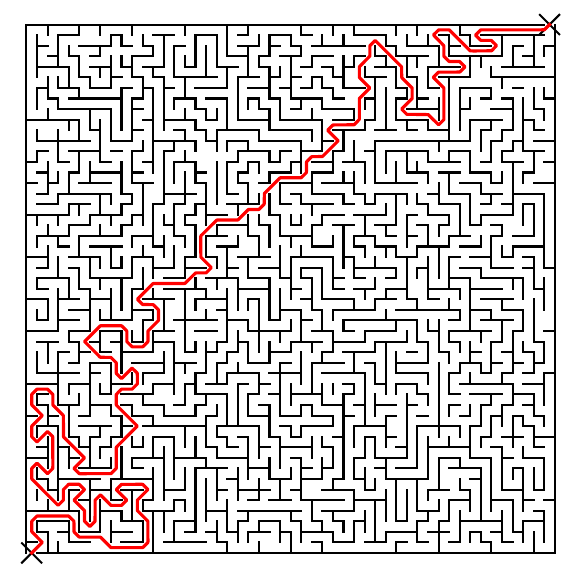}%
\caption{$\epsilon$ = 16}%
\label{m3}%
\end{subfigure}
\begin{subfigure}{.5\columnwidth}
\includegraphics[width=\columnwidth]{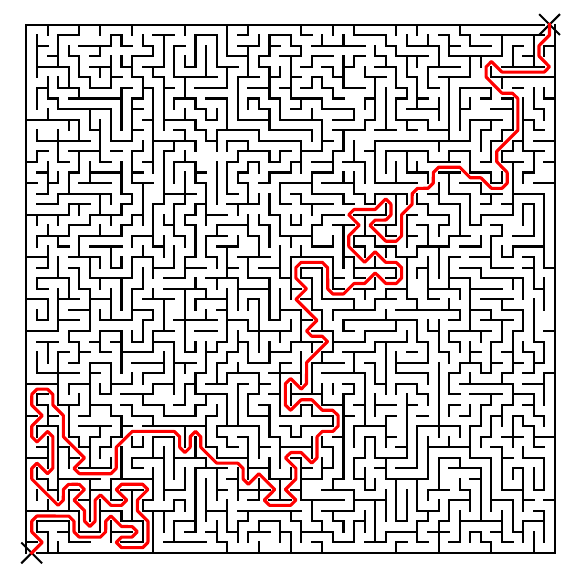}%
\caption{$\epsilon$ = 64}
\label{m4}%
\end{subfigure}%
\caption{Continuous trajectory synthesized by IxG* from one end of the maze to another for different inflations of the heuristic. These inflations are the factor of an upper bound on the cost of the optimal solution.}
\label{fig:mazesubopt}
\end{figure*}

\begin{figure}
    \centering
    \includegraphics[width=\columnwidth]{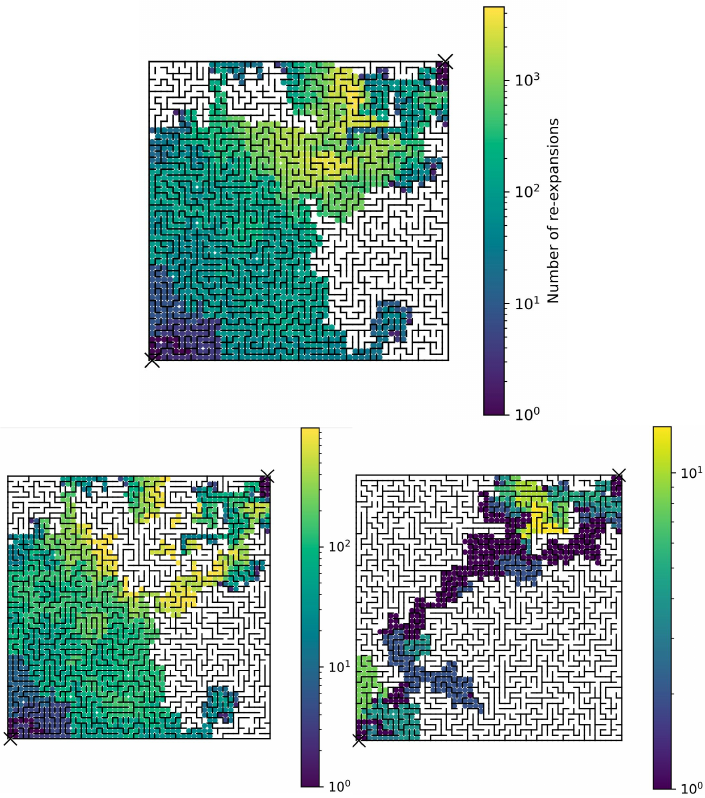}
    \caption{The effect of pruning using the LBG in IxG* as explained in Sec. \ref{sec:ixgs}, \ref{sec:lbg}. The color bar (see different scales for each) denotes the number of re-expansions of a convex region to satisfy the suboptimality factor. The figure at the top is without the pruning operation (\textit{i.e.} $u=\infty$ in Alg. \ref{alg:ixgs}, line \ref{line:ixgs_prune}) for $\epsilon=1$. The bottom figures denote the number of re-expansions with pruning for $\epsilon=1$ (left) and $\epsilon=6$ (right).}
    \label{fig:prune_effect}
\end{figure}

\begin{figure}
\centering
\begin{subfigure}{.49\columnwidth}
\centering
\includegraphics[width=\columnwidth]{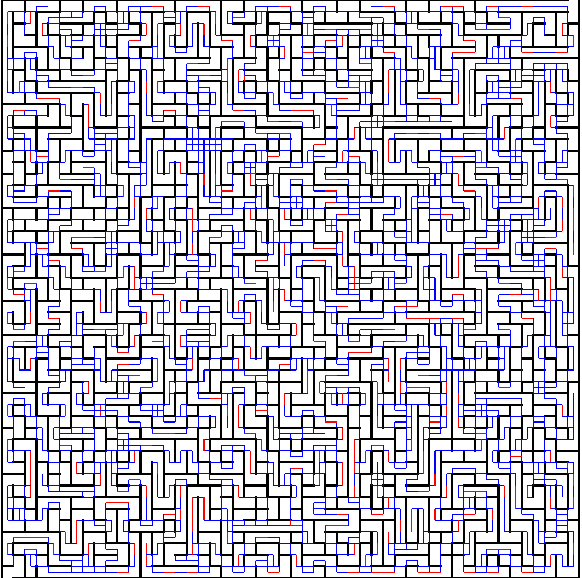}%
\caption{Trajectory is represented as B-splines with order=1.}
\label{fig:maze_lbg1}%
\end{subfigure}
\centering
\begin{subfigure}{.49\columnwidth}
\includegraphics[width=\columnwidth]{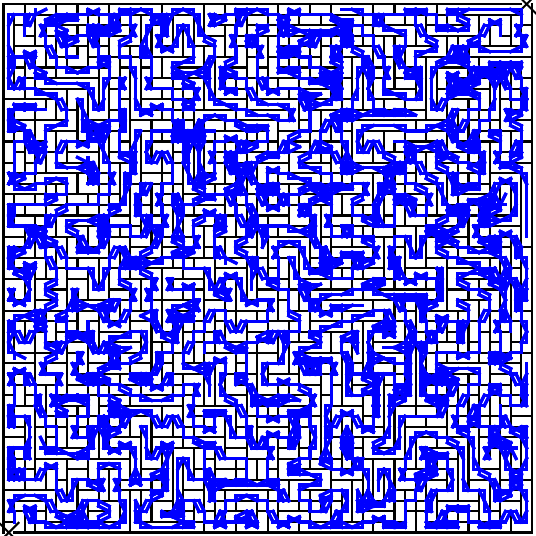}%
\caption{Trajectory is represented as B-splines with order=3.}
\label{fig:maze_lbg2}%
\end{subfigure}
\caption{Lower bound graph (LBG) computed for the 2D maze scenario. The red lines denote zero cost edges.}
\label{fig:maze_lbg}
\end{figure}

For the same pair of start and goal states, we ran IxG* with different inflations of the admissible LBG heuristic to show how the solution changes as we increase the upper bound on the cost of the optimal solution (see Fig. \ref{fig:mazesubopt}). \textcolor{black}{Fig. \ref{fig:varyeps_maze} shows how the solution quality and the runtime of the planner are impacted by $\epsilon$.} As discussed before, allowing re-opening, re-expanding, and revisiting states in IxG* can blow up the complexity of the search. This was controlled by the efficient pruning mechanism using LBG proposed in Sec. \ref{sec:ixgs}. We numerically show the effect of this pruning in Fig. \ref{fig:prune_effect}. The drastic reduction (see different scales in the color bar) in the number of node re-expansions (top vs bottom) implies the tightness of the lower bounded computed by LBG. The LBG used for pruning is provided in Fig. \ref{fig:maze_lbg}. Finally, we plot the effect of suboptimality factor (heuristic inflation) vs planning time in Fig. \ref{fig:hvstime}. As expected, we see that the planning time of IxG* goes down when allowing higher suboptimality of the solution. GCS trajectory optimization takes the same amount of time as there is no way to enforce lower fidelity of solutions. 

\begin{figure}
    \centering
    \includegraphics[width=0.8\columnwidth]{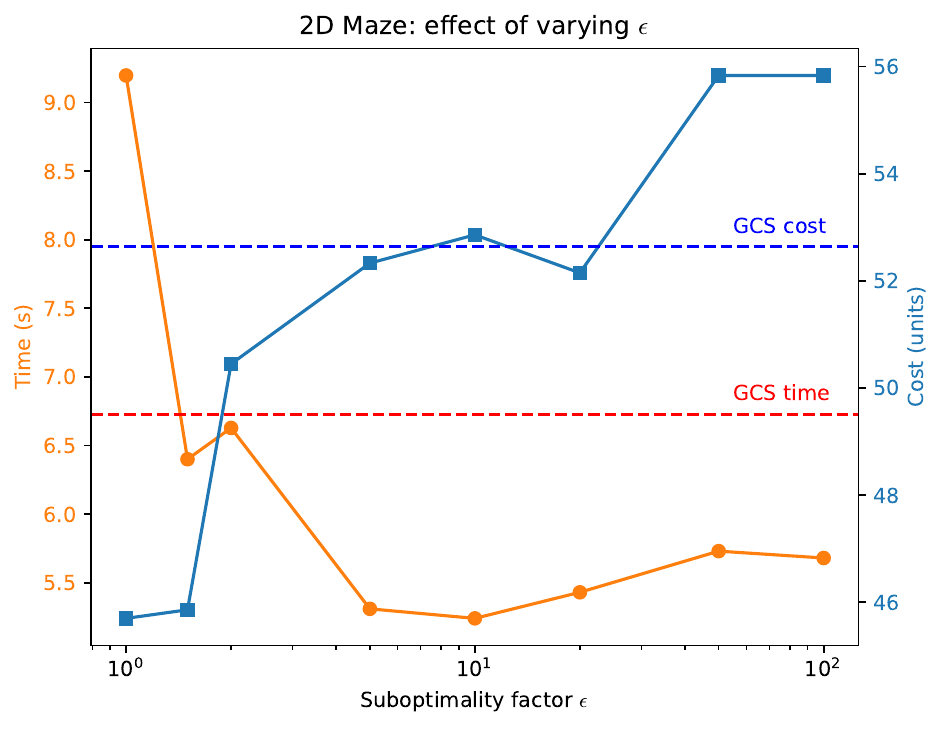}
    \caption{Effect of varying $\epsilon$ on the planning time and the solution cost for 2D maze.}
    \label{fig:varyeps_maze}
\end{figure}


\begin{figure}[htp!]
    \centering
    \includegraphics[width=0.7\columnwidth]{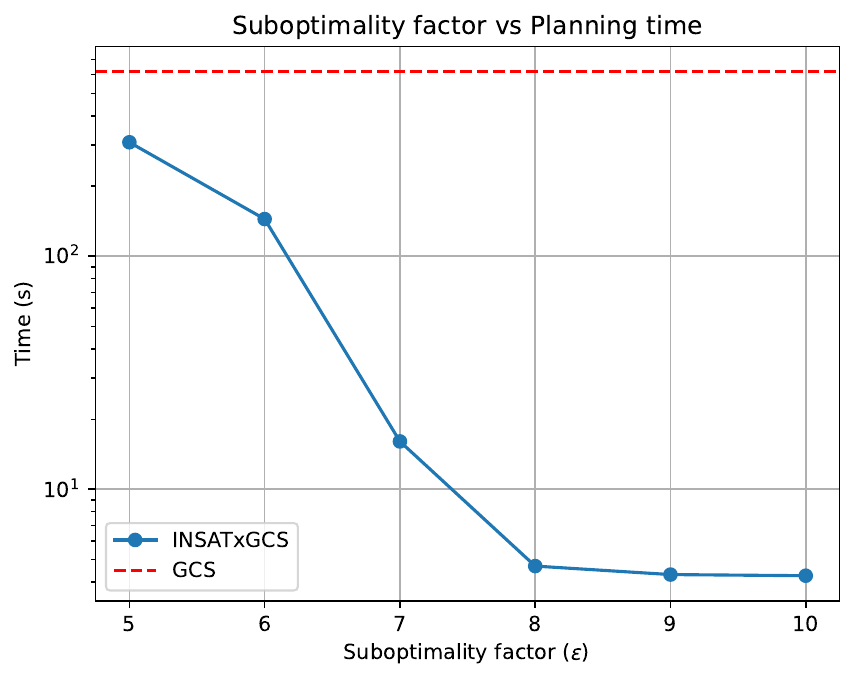}
    \caption{Decrease in planning time of IxG as we increase the suboptimality factor \textcolor{black}{for UAV planning in 50m$\times$50m map}. The GCS planning time is shown using a red dotted line}
    \label{fig:hvstime}
\end{figure}

\begin{table}[htp!]
\centering
\begin{tabular}{c|cc}
\textbf{}                   & \textbf{GCS} & \textbf{IxG*} $(\epsilon=10)$ \\ \hline
\textbf{Success Rate (\%)}  & 100\%              & 100\%              \\
\textbf{Solution Cost}      & 16.055             & 8.676              \\
\textbf{Planning Time (s)}  & 112.867            & 1.85              \\
\textbf{\# Optimized Edges} & 12346              & 211.6             
\end{tabular}
\caption{Statistics showing IxG* outperforming GCS in the 15m$\times$15m UAV forest environment. Note >60x improvement in planning time.}
\label{tab:uav}
\end{table}

\begin{figure}[htp!]
    \centering
    \includegraphics[width=\columnwidth]{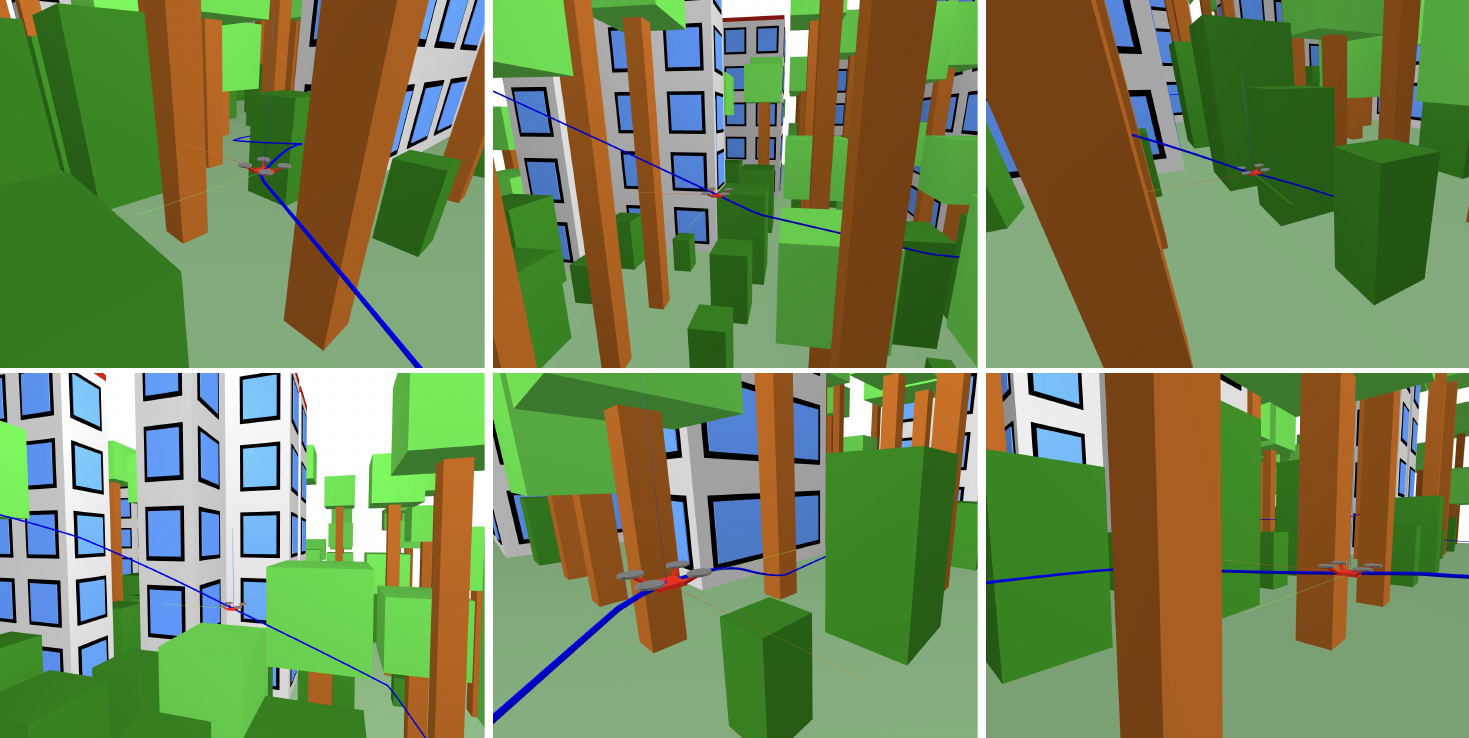}
    \caption{UAV flying through dense forest in simulation using plans generated with INSATxGCS}
    \label{fig:uav}
\end{figure}


\begin{table}[htp!]
\centering
\begin{tabular}{c|cc}
\textbf{}                   & \textbf{IxG*} $(\epsilon=15)$ \\ \hline
\textbf{Success Rate (\%)}  & 96\%              \\
\textbf{Solution Cost (rad)}      & 16.44             \\
\textbf{Planning Time (s)}  & 31.624              \\
\textbf{\# Optimized Edges} & 1832.24            
\end{tabular}
\caption{Results of IxG* for multi-arm manipulation planning. The GCS method was unable to load this problem into memory.}
\label{tab:moto}
\end{table}

\subsection{3D UAV}
We conducted experiments in randomly generated villages, complete with trees and buildings. \textcolor{black}{These maps are the same as what is used in \cite{fpp} and the graph of convex sets is made of axis-aligned 3D boxes and their overlaps.} In a 50m $\times$ 50m map, our environment comprised over 10,000 convex sets and 140,000 edges. In a 15m $\times$ 15m map, we had 900 convex sets and over 12,000 edges. For our experiments, we focused on the 15m $\times$ 15m village, where we compared the proposed IxG* against GCS (Fig. \ref{fig:uav}, Table. \ref{tab:uav}). We refrained from comparing IxG/IxG* with GCS on the 50m $\times$ 50m map. This was due to GCS requiring more than 600 seconds to find solutions, whereas INSATxGCS exhibited an average planning time of just 5.119 seconds. \textcolor{black}{We represented our trajectories using B-spline curves of order 5. We also compared against a recently proposed fast method called FPP for this environment \cite{fpp}. FPP took an average of 2.83s for 15m$\times$15m environment and 5.507s for 50m$\times$50m, thus taking longer than IxG* for both sizes of the map. Nonetheless, it is important to note that \cite{fpp} is applicable only for the graph of convex sets made of axis-aligned boxes and does not have any guarantees on the optimality of the solution unlike IxG*}. 

\subsection{Multi-Arm Manipulation}
The setup consists of three 6-DoF Motoman HC10DTP robot arms positioned around a C-shaped table (Fig. \ref{fig:moto}). The 18-dimensional configuration space is decomposed into 2092 regions using task-specific seeds of interest with a 68\% coverage. \textcolor{black}{With these seeds, the regions are grown using nonlinear programming as explained in \cite{irisnlp} using the implementation in Drake \cite{drake} (\texttt{IrisInConfigurationSpace()}).} The regions had a significant overlap among them leading to 133K edges. For these experiments, we restricted the maximum degree of the graph of convex sets to 8 and pruned significantly overlapping and redundant regions to a total of 67 regions with 2224 edges. We randomly sampled start and goal configurations with a deliberate bias to pick hard combinations that require intricate movements between the arms while being very close to each other. For this experiment, we could not compare against GCS as it was unable \textcolor{black}{to transcribe the whole graph into} costs and constraints \textcolor{black}{and load it entirely} into memory. \textcolor{black}{We even tried GCS with random subgraphs of the full graph used for experiments with IxG* and found that the maximum size of the subgraph that GCS solved had 10 vertices (compared to 67 in the original). Even for a subgraph of size 10, GCS took about 163s to solve and is already 5.2x slower than IxG*.}

IxG* exhibited promising performance, with an average planning time of 31.62s seconds (Table. \ref{tab:moto}). \textcolor{black}{Since IxG* enables informing the search with a suboptimality factor, the quality of the solution can be traded off with the speed of the planner. Fig. \ref{fig:varyeps_mramp} shows how the solution quality and the runtime of the planner are impacted by $\epsilon$. It can be noted that the planning time decreases and solution quality gets worse with increasing suboptimality factor.}

\begin{figure}
    \centering
    \includegraphics[width=0.8\columnwidth]{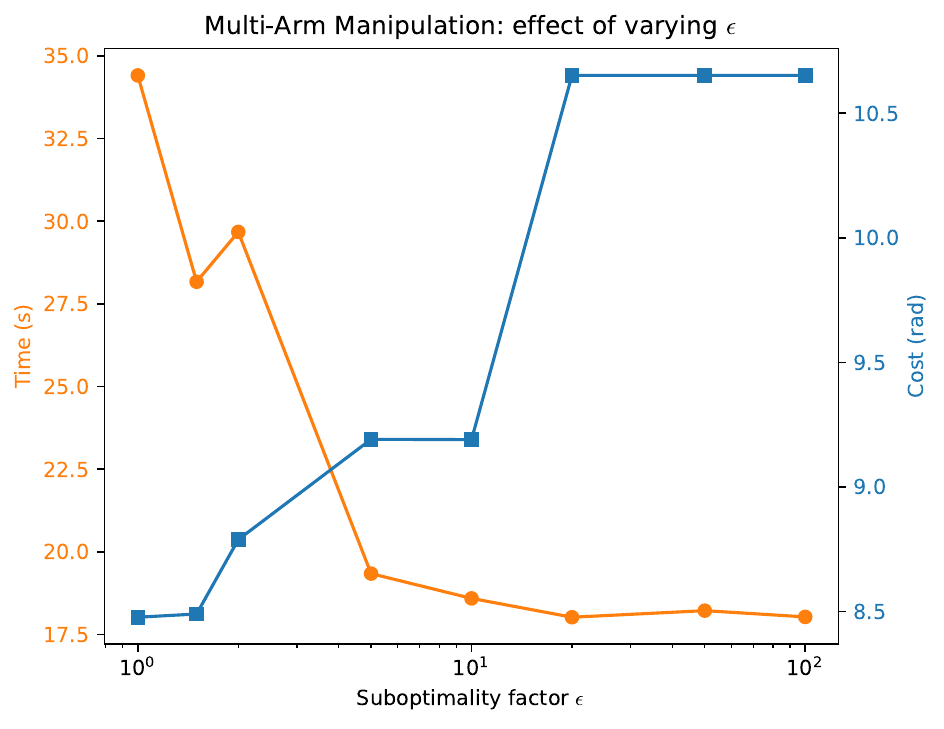}
    \caption{Effect of varying $\epsilon$ on the planning time and the solution cost for the multi-arm manipulation.}
    \label{fig:varyeps_mramp}
\end{figure}

\subsection{Reduction in Complexity}
As a consequence of using implicit search to plan over the graph of convex sets, there is a reduction in the complexity of the planning algorithm across several metrics like reduced size of the problem, a significant drop in the largest optimization call, etc. Some of these metrics were already reported in the above Tables. The size of the GCS graph for different domains are given in Table. \ref{tab:gcssize} along with the giant constant size of the optimization for any query in Table. \ref{tab:gcsoptsize}. As noted, IxG and IxG* solves many small optimization problems and almost never solves one of GCS' size. The largest optimization problem invoked by IxG* is given in Table. \ref{tab:ixgbig}.

\begin{table}[htp!]
\centering
\begin{tabular}{c|cc}
\textbf{}                       & \textbf{\# Convex Sets} & \textbf{\# Edges} \\ \hline
\textbf{2D Maze}                & 2500                   & 5198             \\
\textbf{UAV (15m $\times$ 15m)} & 900                    & 12346            \\
\textbf{UAV (50m $\times$ 50m)} & 10105                  & 140506           \\
\textbf{Motoman HC10DTP 3x Arms}      & 2092                   & 12445            
\end{tabular}
\caption{Details about the graph of convex sets}
\label{tab:gcssize}
\end{table}

\begin{table}[htp!]
\centering
\begin{tabular}{c|ccc}
                    & \textbf{\# Decision Vars} & \textbf{\# Costs} & \textbf{\# Constraints} \\ \hline
\textbf{2D Maze}       & 119586                    & 15202             & 256261                  \\
\textbf{UAV (15m $\times$ 15m)} & 555657                    & 17751             & 759756                  \\
\textbf{UAV (50m $\times$ 50m)} & 6322851                   & 201141            & 8642649                
\end{tabular}
\caption{GCS optimization problem size}
\label{tab:gcsoptsize}
\end{table}

\begin{table}[htp!]
\begin{tabular}{c|ccl}
& \textbf{2D Maze} & \textbf{UAV (15m $\times$ 15m)} & \textbf{UAV (50m $\times$ 50m)} \\ \hline
\textbf{\begin{tabular}[c]{@{}c@{}}Max \# \\ Decision \\ Vars\end{tabular}} & 1845             & 196                             & \multicolumn{1}{c}{1175}       
\end{tabular}
\caption{Largest optimization IxG* invoked}
\label{tab:ixgbig}
\end{table}

\subsection{Implementation Details}
An important implementation detail that contributes to the efficiency of IxG and IxG* lies in how the costs and constraints are built for the optimization. Though the graph is explored and the optimization is solved implicitly, for a given decomposition of the graph, the symbolic costs and constraints of all the GCS vertices and edges over the entire graph are constructed and cached offline. During runtime, when the search explores a particular path in the GCS graph, the cached symbolic costs and constraints are retrieved at virtually no cost and supplied to the convex solver. 

\section{Conclusion}
We showed that using INSAT to synthesize trajectories over the graph of convex sets is a better alternative compared to the GCS batch optimization. To that end, we developed two algorithms IxG and IxG*, and discussed trivial extensions to parallelized IxG* without sacrificing any of the theoretical properties. We can plan orders of magnitude faster with better theoretical properties. Graphs of convex sets tackled one of the biggest bottlenecks of collision checking in motion planning by representing the free planning space as a union of convex sets stored as a graph. This representation can make motion planning dramatically simple as every pair of points in a convex set are connected by a straight line and as optimization over convex sets is extremely well studied. However joint optimization over the entire graph is unnecessary in most instances and can limit the capability of the GCS representation. Using a search-based method also brings the advantage of all the planning techniques developed in discrete graph search. For example, there are many replanning and anytime algorithms in graph search that efficiently reuse search efforts in dynamically changing environments or return a quick time-bounded solution while improving the solution quality. We believe that the introduction of a graph search-based framework such as INSAT to explore and prune the convex sets while optimizing the trajectory via them widens the scope and application of GCS.

\section{Acknowledgements}
This work was supported by grants W911NF-21-1-0050 and W911NF-18-2-0218 of the ARL-sponsored A2I2 program. We thank Yorai Shaoul for his help in rendering some of the visualizations presented in this work.



\bibliographystyle{plainnat}
\bibliography{references}

\begin{thebibliography}{38}
\providecommand{\natexlab}[1]{#1}
\providecommand{\url}[1]{\texttt{#1}}
\expandafter\ifx\csname urlstyle\endcsname\relax
  \providecommand{\doi}[1]{doi: #1}\else
  \providecommand{\doi}{doi: \begingroup \urlstyle{rm}\Url}\fi

\bibitem[Alwala and Mukadam(2021)]{so2}
Kalyan~Vasudev Alwala and Mustafa Mukadam.
\newblock Joint sampling and trajectory optimization over graphs for online motion planning.
\newblock In \emph{2021 IEEE/RSJ International Conference on Intelligent Robots and Systems (IROS)}, pages 4700--4707. IEEE, 2021.

\bibitem[Amice et~al.(2022)Amice, Dai, Werner, Zhang, and Tedrake]{ciris}
Alexandre Amice, Hongkai Dai, Peter Werner, Annan Zhang, and Russ Tedrake.
\newblock Finding and optimizing certified, collision-free regions in configuration space for robot manipulators.
\newblock In \emph{International Workshop on the Algorithmic Foundations of Robotics}, pages 328--348. Springer, 2022.

\bibitem[Betts(2010)]{optctr}
John~T Betts.
\newblock \emph{Practical methods for optimal control and estimation using nonlinear programming}.
\newblock SIAM, 2010.

\bibitem[Choudhury et~al.(2016)Choudhury, Gammell, Barfoot, Srinivasa, and Scherer]{rabit}
Sanjiban Choudhury, Jonathan~D Gammell, Timothy~D Barfoot, Siddhartha~S Srinivasa, and Sebastian Scherer.
\newblock Regionally accelerated batch informed trees (rabit*): A framework to integrate local information into optimal path planning.
\newblock In \emph{2016 IEEE International Conference on Robotics and Automation (ICRA)}, pages 4207--4214. IEEE, 2016.

\bibitem[Dasgupta et~al.(2006)Dasgupta, Papadimitriou, and Algorithms]{algtb}
Sanjoy Dasgupta, CH~Papadimitriou, and Umesh~Vazirani Algorithms.
\newblock Mcgraw-hill science.
\newblock \emph{Engineering/Math}, 2006.

\bibitem[Deits and Tedrake(2015)]{iris1}
Robin Deits and Russ Tedrake.
\newblock Computing large convex regions of obstacle-free space through semidefinite programming.
\newblock In \emph{Algorithmic Foundations of Robotics XI: Selected Contributions of the Eleventh International Workshop on the Algorithmic Foundations of Robotics}, pages 109--124. Springer, 2015.

\bibitem[Diehl et~al.(2006)Diehl, Bock, Diedam, and Wieber]{fasttrajopt}
Moritz Diehl, Hans~Georg Bock, Holger Diedam, and P-B Wieber.
\newblock Fast direct multiple shooting algorithms for optimal robot control.
\newblock \emph{Fast motions in biomechanics and robotics: optimization and feedback control}, pages 65--93, 2006.

\bibitem[Gammell et~al.(2015)Gammell, Srinivasa, and Barfoot]{bit}
Jonathan~D Gammell, Siddhartha~S Srinivasa, and Timothy~D Barfoot.
\newblock Batch informed trees (bit*): Sampling-based optimal planning via the heuristically guided search of implicit random geometric graphs.
\newblock In \emph{2015 IEEE international conference on robotics and automation (ICRA)}, pages 3067--3074. IEEE, 2015.

\bibitem[Hart et~al.(1968)Hart, Nilsson, and Raphael]{astar}
Peter~E Hart, Nils~J Nilsson, and Bertram Raphael.
\newblock A formal basis for the heuristic determination of minimum cost paths.
\newblock \emph{IEEE transactions on Systems Science and Cybernetics}, 4\penalty0 (2):\penalty0 100--107, 1968.

\bibitem[Hauser and Zhou(2016)]{asymprrt}
Kris Hauser and Yilun Zhou.
\newblock Asymptotically optimal planning by feasible kinodynamic planning in a state--cost space.
\newblock \emph{IEEE Transactions on Robotics}, 32\penalty0 (6):\penalty0 1431--1443, 2016.

\bibitem[Kamat et~al.(2022)Kamat, Ortiz-Haro, Toussaint, Pokorny, and Orthey]{so1}
Jay Kamat, Joaquim Ortiz-Haro, Marc Toussaint, Florian~T Pokorny, and Andreas Orthey.
\newblock Bitkomo: Combining sampling and optimization for fast convergence in optimal motion planning.
\newblock In \emph{2022 IEEE/RSJ International Conference on Intelligent Robots and Systems (IROS)}, pages 4492--4497. IEEE, 2022.

\bibitem[Kavraki et~al.(1996)Kavraki, Svestka, Latombe, and Overmars]{prm}
Lydia~E Kavraki, Petr Svestka, J-C Latombe, and Mark~H Overmars.
\newblock Probabilistic roadmaps for path planning in high-dimensional configuration spaces.
\newblock \emph{IEEE tran. on Robot. Autom.}, 12\penalty0 (4):\penalty0 566--580, 1996.

\bibitem[Kunz and Stilman(2015)]{completerrt}
Tobias Kunz and Mike Stilman.
\newblock Kinodynamic rrts with fixed time step and best-input extension are not probabilistically complete.
\newblock In \emph{Algorithmic Foundations of Robotics XI: Selected Contributions of the Eleventh International Workshop on the Algorithmic Foundations of Robotics}, pages 233--244. Springer, 2015.

\bibitem[LaValle and Kuffner~Jr(2001)]{kinodynamic}
Steven~M LaValle and James~J Kuffner~Jr.
\newblock Randomized kinodynamic planning.
\newblock \emph{Int. J. Robot. Research}, 20\penalty0 (5):\penalty0 378--400, 2001.

\bibitem[LaValle et~al.(1998)]{rrt}
Steven~M LaValle et~al.
\newblock Rapidly-exploring random trees: A new tool for path planning.
\newblock 1998.

\bibitem[Likhachev and Ferguson(2009)]{maxprimitives}
Maxim Likhachev and Dave Ferguson.
\newblock Planning long dynamically feasible maneuvers for autonomous vehicles.
\newblock \emph{The International Journal of Robotics Research}, 28\penalty0 (8):\penalty0 933--945, 2009.

\bibitem[Littlefield and Bekris(2018)]{ss2}
Zakary Littlefield and Kostas~E Bekris.
\newblock Efficient and asymptotically optimal kinodynamic motion planning via dominance-informed regions.
\newblock In \emph{2018 IEEE/RSJ International Conference on Intelligent Robots and Systems (IROS)}, pages 1--9. IEEE, 2018.

\bibitem[Marcucci et~al.(2021)Marcucci, Umenberger, Parrilo, and Tedrake]{sppgcs}
Tobia Marcucci, Jack Umenberger, Pablo~A Parrilo, and Russ Tedrake.
\newblock Shortest paths in graphs of convex sets.
\newblock \emph{arXiv preprint arXiv:2101.11565}, 2021.

\bibitem[Marcucci et~al.(2023{\natexlab{a}})Marcucci, Nobel, Tedrake, and Boyd]{fpp}
Tobia Marcucci, Parth Nobel, Russ Tedrake, and Stephen Boyd.
\newblock Fast path planning through large collections of safe boxes.
\newblock \emph{arXiv preprint arXiv:2305.01072}, 2023{\natexlab{a}}.

\bibitem[Marcucci et~al.(2023{\natexlab{b}})Marcucci, Petersen, von Wrangel, and Tedrake]{gcs}
Tobia Marcucci, Mark Petersen, David von Wrangel, and Russ Tedrake.
\newblock Motion planning around obstacles with convex optimization.
\newblock \emph{Science robotics}, 8\penalty0 (84):\penalty0 eadf7843, 2023{\natexlab{b}}.

\bibitem[Natarajan et~al.(2021{\natexlab{a}})Natarajan, Choset, and Likhachev]{insat}
Ramkumar Natarajan, Howie Choset, and Maxim Likhachev.
\newblock Interleaving graph search and trajectory optimization for aggressive quadrotor flight.
\newblock \emph{IEEE Robotics and Automation Letters}, 6\penalty0 (3):\penalty0 5357--5364, 2021{\natexlab{a}}.

\bibitem[Natarajan et~al.(2021{\natexlab{b}})Natarajan, Choset, and Likhachev]{insat1}
Ramkumar Natarajan, Howie Choset, and Maxim Likhachev.
\newblock Interleaving graph search and trajectory optimization for aggressive quadrotor flight.
\newblock \emph{IEEE Robotics and Automation Letters}, 6\penalty0 (3):\penalty0 5357--5364, 2021{\natexlab{b}}.
\newblock \doi{10.1109/LRA.2021.3067298}.

\bibitem[Natarajan et~al.(2022)Natarajan, Johnston, Simaan, Likhachev, and Choset]{insat2}
Ramkumar Natarajan, Garrison L.~H. Johnston, Nabil Simaan, Maxim Likhachev, and Howie Choset.
\newblock Torque-limited manipulation planning through contact by interleaving graph search and trajectory optimization, 2022.

\bibitem[Natarajan et~al.(2023{\natexlab{a}})Natarajan, Johnston, Simaan, Likhachev, and Choset]{troptc}
Ramkumar Natarajan, Garrison~L Johnston, Nabil Simaan, Maxim Likhachev, and Howie Choset.
\newblock Long-horizon torque-limited planning through contact using discrete search and continuous optimization.
\newblock In \emph{IROS 2023 Workshop on Leveraging Models for Contact-Rich Manipulation}, 2023{\natexlab{a}}.

\bibitem[Natarajan et~al.(2023{\natexlab{b}})Natarajan, Johnston, Simaan, Likhachev, and Choset]{insatptc}
Ramkumar Natarajan, Garrison~LH Johnston, Nabil Simaan, Maxim Likhachev, and Howie Choset.
\newblock Torque-limited manipulation planning through contact by interleaving graph search and trajectory optimization.
\newblock In \emph{2023 IEEE International Conference on Robotics and Automation (ICRA)}, pages 8148--8154. IEEE, 2023{\natexlab{b}}.

\bibitem[Natarajan et~al.(2024{\natexlab{a}})Natarajan, Mukherjee, Choset, and Likhachev]{pinsat}
Ramkumar Natarajan, Shohin Mukherjee, Howie Choset, and Maxim Likhachev.
\newblock Pinsat: Parallelized interleaving of graph search and trajectory optimization for kinodynamic motion planning.
\newblock \emph{arXiv preprint arXiv:2401.08948}, 2024{\natexlab{a}}.

\bibitem[Natarajan et~al.(2024{\natexlab{b}})Natarajan, Yang, Xie, Oza, Das, Islam, Saleem, Choset, and Likhachev]{insatshield}
Ramkumar Natarajan, Hanlan Yang, Qintong Xie, Yash Oza, Manash~Pratim Das, Fahad Islam, Muhammad~Suhail Saleem, Howie Choset, and Maxim Likhachev.
\newblock Preprocessing-based kinodynamic motion planning framework for intercepting projectiles using a robot manipulator.
\newblock \emph{arXiv preprint arXiv:2401.08022}, 2024{\natexlab{b}}.

\bibitem[Ortiz-Haro et~al.(2023)Ortiz-Haro, Hoenig, Hartmann, and Toussaint]{idba}
Joaquim Ortiz-Haro, Wolfgang Hoenig, Valentin~N Hartmann, and Marc Toussaint.
\newblock idb-a*: Iterative search and optimization for optimal kinodynamic motion planning.
\newblock \emph{arXiv preprint arXiv:2311.03553}, 2023.

\bibitem[Petersen and Tedrake(2023)]{irisnlp}
Mark Petersen and Russ Tedrake.
\newblock Growing convex collision-free regions in configuration space using nonlinear programming.
\newblock \emph{arXiv preprint arXiv:2303.14737}, 2023.

\bibitem[Pohl(1970)]{pohlwastar}
Ira Pohl.
\newblock Heuristic search viewed as path finding in a graph.
\newblock \emph{Artificial intelligence}, 1\penalty0 (3-4):\penalty0 193--204, 1970.

\bibitem[Ratliff et~al.(2009)Ratliff, Zucker, Bagnell, and Srinivasa]{chomp}
Nathan Ratliff, Matt Zucker, J~Andrew Bagnell, and Siddhartha Srinivasa.
\newblock Chomp: Gradient optimization techniques for efficient motion planning.
\newblock In \emph{2009 IEEE international conference on robotics and automation}, pages 489--494. IEEE, 2009.

\bibitem[Sakcak et~al.(2019)Sakcak, Bascetta, Ferretti, and Prandini]{ss1}
Basak Sakcak, Luca Bascetta, Gianni Ferretti, and Maria Prandini.
\newblock Sampling-based optimal kinodynamic planning with motion primitives.
\newblock \emph{Autonomous Robots}, 43\penalty0 (7):\penalty0 1715--1732, 2019.

\bibitem[Schulman et~al.(2014)Schulman, Duan, Ho, Lee, Awwal, Bradlow, Pan, Patil, Goldberg, and Abbeel]{trajopt}
John Schulman, Yan Duan, Jonathan Ho, Alex Lee, Ibrahim Awwal, Henry Bradlow, Jia Pan, Sachin Patil, Ken Goldberg, and Pieter Abbeel.
\newblock Motion planning with sequential convex optimization and convex collision checking.
\newblock \emph{The International Journal of Robotics Research}, 33\penalty0 (9):\penalty0 1251--1270, 2014.

\bibitem[Tassa et~al.(2014)Tassa, Mansard, and Todorov]{ddp}
Yuval Tassa, Nicolas Mansard, and Emo Todorov.
\newblock Control-limited differential dynamic programming.
\newblock In \emph{2014 IEEE International Conference on Robotics and Automation (ICRA)}, pages 1168--1175. IEEE, 2014.

\bibitem[Tedrake and the Drake Development~Team(2019)]{drake}
Russ Tedrake and the Drake Development~Team.
\newblock Drake: Model-based design and verification for robotics, 2019.
\newblock URL \url{https://drake.mit.edu}.

\bibitem[Toussaint(2017)]{komo}
Marc Toussaint.
\newblock A tutorial on newton methods for constrained trajectory optimization and relations to slam, gaussian process smoothing, optimal control, and probabilistic inference.
\newblock \emph{Geometric and numerical foundations of movements}, pages 361--392, 2017.

\bibitem[Werner et~al.(2023)Werner, Amice, Marcucci, Rus, and Tedrake]{vcc}
Peter Werner, Alexandre Amice, Tobia Marcucci, Daniela Rus, and Russ Tedrake.
\newblock Approximating robot configuration spaces with few convex sets using clique covers of visibility graphs.
\newblock \emph{arXiv preprint arXiv:2310.02875}, 2023.

\bibitem[Zhang et~al.(2020)Zhang, Liniger, and Borrelli]{collopt}
Xiaojing Zhang, Alexander Liniger, and Francesco Borrelli.
\newblock Optimization-based collision avoidance.
\newblock \emph{IEEE Transactions on Control Systems Technology}, 29\penalty0 (3):\penalty0 972--983, 2020.

\end{thebibliography}

\clearpage

\section{APPENDIX}
In the following sections, some of the definitions are repeated from the main portion of the paper for self-containedness.

\section*{Appendix 1}
\begin{table}[h!]
\centering
\begin{tabular}{c|ccc}
\textbf{} & \textbf{Optimality} & \textbf{Bounded Suboptimality} & \textbf{Completeness} \\ \hline
\textbf{IxG}  & \xmark  & \xmark & \checkmark\tablefootnote{Under the assumption that all the edge constraints can be satisfied}              \\
\textbf{IxG*} & \checkmark & \checkmark & \checkmark             \\
\end{tabular}
\caption*{Optimality and completeness of IxG and IxG*.}
\label{tab:tprop_appx}
\end{table}

\begin{definition}
    \textbf{Markov Property in Search:} It states that the cost and the set of successors of a state depend \textit{only} on the current state and not on the history of the path leading up to it.
\end{definition}

As mentioned before, in the case of planning over a graph of convex sets, the graph is explicitly provided. So generating the set of successors of a state does not violate the Markov property. However, computing the cost of the successor requires computing the edge/trajectory and satisfying constraints which itself could depend on the path of convex sets leading up to the successor. This breaks the Markov property in IxG. We alleviate this in IxG* by allowing re-expansions with duplicates

\begin{assump}
    There exists a path in $G_{\gcsnode}$ such that it contains $q_{0T}^*(t)$ minimizing Eq. \ref{eq:pobj}.
\end{assump}
\subsection{Properties of IxG*}
\begin{theorem}
    \textbf{Optimality of IxG*:} If Assumption 1 holds, then on expanding $\gcsnode^{0\ldots T}$, \textit{i.e.} when \textproc{Key}($\gcsnode^{0\ldots T}$) $>$ OPEN.min() (Alg. \ref{alg:ixgs}, line \ref{line:ixgs_term}), $\gcsnode^{0\ldots T}$.trajectory() will return $q_{0T}^*(t)$, where $q_{0T}^*(t)$ is the global minimizer of Eq. \ref{eq:pobj}. Since this is proof of optimality, $\epsilon=1$ in Alg. \ref{alg:ixgs}.
\end{theorem}
\begin{proof}
    The proof is by induction. Let us assume that $\forall \gcsnode^{0\ldots b}$ which are expanded, $\ g(\gcsnode^{0\ldots b}) = g^*(\gcsnode^{0\ldots b})$. This is trivially true for $\gcsnode_{0}$. We know that OPEN separates parts of the path from $\gcsnode_0$ to goal $\gcsnode_T$ that are expanded from the parts of the path that are never seen. The next state to be expanded is given by (Alg. \ref{alg:ixgs}, line \ref{line:ixgs_pop} and \ref{line:key})
    $$\gcsnode^{0\ldots c} = \argmin_{\gcsnode^{0\ldots d} \in \text{OPEN}} g(\gcsnode^{0\ldots d}) + l(\gcsnode_d).$$
    Let us assume $g(\gcsnode^{0\ldots d})$ is suboptimal. Then there must be at least one path $\gcsnode^{0\ldots e}$ in OPEN that is part of the optimal path from $0$ to $d$ on GCS which contains the optimal trajectory $q_{0d}^*(t)$. So
    $$g(\gcsnode^{0\ldots e}) + l(\gcsnode_e) \ge g(\gcsnode^{0\ldots c}) + l(\gcsnode_c)$$
    But 
    $$g(\gcsnode^{0\ldots e}) + c((\gcsnode^{0\ldots c}, \gcsnode^{c\ldots e})) < g(\gcsnode^{0\ldots c})$$
    $$\Rightarrow g(\gcsnode^{0\ldots e}) + c((\gcsnode^{0\ldots c}, \gcsnode^{c\ldots e})) + l(\gcsnode_e) < g(\gcsnode^{0\ldots c})  + l(\gcsnode_c)$$
    $$\Rightarrow g(\gcsnode^{0\ldots e}) + l(\gcsnode_e) < g(\gcsnode^{0\ldots c})  + l(\gcsnode_c)$$
    The above step is a contradiction to our assumption that $g(\gcsnode^{0\ldots d})$ is suboptimal. Hence it must be the case that $g(\gcsnode^{0\ldots d})$ is optimal upon expansion. 
\end{proof}

\begin{theorem}
    \textbf{Bounded suboptimality of IxG*:} If Assumption 1 holds, using $\epsilon > 1$ to prune duplicates under the criteria $c(q_{0s}(t)) + \epsilon*l(\gcsnode_s) > \epsilon*u$ (Alg. \ref{alg:ixgs}, line \ref{line:ixgs_prune}), the termination condition \textproc{Key}($\gcsnode^{0\ldots T}$) $>$ OPEN.min() (Alg. \ref{alg:ixgs}, line \ref{line:ixgs_term}), will return trajectory $q_{0T}^\sim(t)$, whose cost $c(q_{0T}^\sim(t)) \le \epsilon*c(q_{0T}^*(t))$ where $q_{0T}^*(t)$ is the global minimizer of Eq. \ref{eq:pobj}.
\end{theorem}
\begin{proof}
    The proof follows from the above proof of optimality applied to paths instead of vertices in the bounded suboptimality proof of wA* \cite{pohlwastar}.
\end{proof}



\section*{Appendix 2}
\subsection{Properties of Lower Bound Graph (LBG)}
The set of convex sets $\gcsnode = \{\gcsnode_1, \gcsnode_2, \ldots, \gcsnode_n\}\subset \mathbb{R}^d$ are represented as a graph $G_\gcsnode = (V_\gcsnode, E_\gcsnode)$. Some of the properties of the lower bound graph $G_{lb} = (V_{lb}, E_{lb})$ include 
\begin{itemize}
    \item The number of vertices is bounded by 
    $$\mid V_{lb}\mid \le 2\sum_{i=1}^{n} |E_{\gcsnode_i}^{in}||E_{\gcsnode_i}^{out}|$$
    where $E_{\gcsnode_i}^{in}$ and $E_{\gcsnode_i}^{out}$ are the incoming edges to and outgoing edges from vertex $\gcsnode_i$. 
    \item The number of edges is bounded by
    $$\mid E_{lb}\mid \le \frac{\mid V_{lb}\mid}{2} + \sum_{e_\gcsnode \in E_\gcsnode} \mid V_{lb}^{e_\gcsnode} \mid^2 $$
    where $V_{lb}^{e_\gcsnode} = \{ v_{lb}\in V_{lb} \mid v_{lb} \in e_\gcsnode\}$. An example of $v_{lb} \in e_\gcsnode$ when the convex sets are overlapping can be $v_{lb} \in \gcsnode_u \cap \gcsnode_v$ where $e_\gcsnode = (\gcsnode_u, \gcsnode_v)$.
    \item The degree of LBG is 
    $$deg(G_{lb}) = \max_{e_\gcsnode \in E_\gcsnode} \mid V_{lb}^{e_\gcsnode} \mid $$
\end{itemize}

Let $\lbgtogcs:V_{lb} \rightarrow V_\gcsnode$ denote a many-to-one mapping from the vertices in $G_{lb}$ to the vertices in $G_\gcsnode$. The mapping $\lbgtogcs^{-1}$ can be easily aggregated by keeping track of $v_\gcsnode \in V_\gcsnode$ for every $\gcsnode_i \in V_\gcsnode$ when constructing LBG (Alg. \ref{alg:lbg}).

\begin{theorem}
    Consider a sequence of convex sets $\gcsnode_{1\ldots K} = (\gcsnode_1, \gcsnode_2, \ldots, \gcsnode_K) \in V_\gcsnode$ connected by edges from $E_\gcsnode$. Let $q_{1K}^*(t)$ be the optimal trajectory via the sequence satisfying Eq. \ref{eq:pobj} and $G_{lb}^{1\ldots K} = (V_{lb}^{1\ldots K}, E_{lb}^{1\ldots K})$ be a subgraph of $G_{lb}$ corresponding to this sequence $\gcsnode_{1\ldots K}$. Let $v_{lb}^k \in \lbgtogcs^{-1}(\gcsnode_k)$. Then the cost of the optimal path $p_{1K}^*$ from $v_{lb}^1$ to $v_{lb}^K$ in the subgraph $G_{lb}^{1\ldots K}$ underestimates $q_{1K}^*(t)$ 
    $$c(p_{1K}^*) \le c(q_{1K}^*(t))$$
\end{theorem}
\begin{proof}
    Given a triplet sequence of convex sets, Alg. \ref{alg:lbg} constructs the LBG edges by solving the same minimization as the IxG* (Eq. \ref{eq:pobj}) with a subset of constraints. For example, LBG construction does not have to satisfy some of the edge constraints of GCS such as continuity or smoothness. Therefore, by satisfying fewer constraints for the same sequence of convex sets, the cost of LBG solution is upper-bounded by the optimal solution in GCS.
\end{proof}

\begin{theorem}
    Consider $v_{lb}^s, v_{lb}^t \in V_{lb}$. Let $\gcsnode_s = \lbgtogcs(v_{lb}^s)$ and $\gcsnode_t = \lbgtogcs(v_{lb}^t)$. Let the optimal trajectory from $\gcsnode_s$ to $\gcsnode_t$ that satisfies Eq. \ref{eq:pobj} be $q_{st}^*(t)$. Then 
    $$c(r_{st}^*) \le c(q_{st}^*(t))$$
    where $r_{st}^*$ is the optimal path on $G_{lb}$ from $v_{lb}^s$ to $v_{lb}^t$.
\end{theorem}
\begin{proof}
From the previous theorem, we know that the optimal path $p^*_{st}$ in the subgraph $G_{lb}^{s\ldots t}$ holds
$$c(p_{st}^*) \le c(q_{st}^*(t))$$
Since $G_{lb}^{s\ldots t}$ is a subgraph of $G_{lb}$
$$c(r_{st}^*) \le c(p_{st}^*)$$
Thus 
$$c(r_{st}^*) \le c(q_{st}^*(t))$$
\end{proof}

    

\end{document}